\documentclass[11pt]{article}

\usepackage{hyperref}
\usepackage{array}
\usepackage{algorithm}
\usepackage[noend]{algorithmic}
\usepackage{amsfonts,amsmath, amssymb, bm}
\usepackage{natbib}% for citations
\def\math#1{$#1$}

\def\mand#1{$$#1$$}

\def\frac#1#2{{#1\over #2}}

\def\mld#1{\begin{equation}
#1
\end{equation}}
\def\eqar#1{\begin{eqnarray}
#1
\end{eqnarray}}
\def\eqan#1{\begin{eqnarray*}
#1
\end{eqnarray*}}

\def\cl#1{{\cal #1}}

\def\x{{\mathbf x}}
\def\y{{\mathbf y}}

\def\r#1{{(\ref{#1})}}

\def\dotfil{\leaders\hbox to 1.5mm{.}\hfill}
\newcounter{rmnum}
\def\RN#1{\setcounter{rmnum}{#1}\uppercase\expandafter{\romannumeral\value{rmnum}}}
\def\rn#1{\setcounter{rmnum}{#1}\expandafter{\romannumeral\value{rmnum}}}

\newcommand{\ONorm }[1]{\mbox{}\left\|#1\right\|_{\ell_1}  }

\newcommand{\ZNorm }[1]{\mbox{}\left\|#1\right\|_0  }

\newcommand{\FNorm }[1]{\mbox{}\left\|#1\right\|_F  }
\newcommand{\FNormS}[1]{\mbox{}\left\|#1\right\|_F^2}

\newcommand{\TNorm}[1]{\mbox{}\left\|#1\right\|_2}
\newcommand{\TsNorm}[1]{\mbox{}\|#1\|_2}
\newcommand{\TNormS}[1]{\mbox{}\left\|#1\right\|_2^2}

\newcommand{\setlinespacing}[1]%
           {\setlength{\baselineskip}{#1 \defbaselineskip}}

\newcommand{\abs }[1]{\left|#1\right|}

\newtheorem{lemma}{Lemma}
\newtheorem{theorem}{Theorem}

\newenvironment{proof}{\noindent {\em Proof:}}{\hspace*{\fill}\mbox{$\diamond$}}

\newcommand{\mat}[1]{{\ensuremath{\bm{\mathrm{#1}}}}}

\def\x{{\mathbf x}}

\def\matA{\mat{A}}

\def\matI{\mat{I}}

\def\matS{\mat{S}}

\def\matU{\mat{U}}
\def\matV{\mat{V}}

\def\matX{\mat{X}}

\def\matSig{\mat{\Sigma}}

\def\vv{\mathbf{v}}
\def\Exp{\mathbb{E}}

\DeclareMathSymbol{\R}{\mathbin}{AMSb}{"52}
\usepackage{graphicx}
\usepackage{epstopdf}
\usepackage{float}
\usepackage{caption}
\usepackage{subcaption}
\usepackage{times}
\usepackage{multirow}

\begin{document}

\title{Approximating Sparse PCA from Incomplete Data}
%
%\iffalse
\author{
Abhisek Kundu
\thanks{
Department of Computer Science,
Rensselaer Polytechnic Institute,
Troy, NY,
kundua2@rpi.edu.
}
\and
Petros Drineas
\thanks{
Department of Computer Science,
Rensselaer Polytechnic Institute,
Troy, NY,
drinep@cs.rpi.edu.
}
\and
Malik Magdon-Ismail
\thanks{
Department of Computer Science,
Rensselaer Polytechnic Institute,
Troy, NY,
magdon@cs.rpi.edu.
}
}
%\fi

\date{}
\maketitle

\begin{abstract}
\noindent 
We study how well one can recover sparse principal components
of a data matrix 
using a sketch formed from a few of its elements. 
We show that for a wide class of optimization problems,
if the sketch is close (in the spectral norm) to the original data
matrix, then one can recover a near optimal solution to the optimization
problem by using the sketch. In particular, we use this approach to
obtain sparse principal components and show that for \math{m} data points
in \math{n} dimensions,
\math{O(\epsilon^{-2}\tilde k\max\{m,n\})} elements gives an
\math{\epsilon}-additive approximation to the sparse PCA problem
(\math{\tilde k} is the stable rank of the data matrix).
We demonstrate our algorithms extensively
on image, text, biological and financial data.
The results show that not only are we able to recover the sparse PCAs from 
the incomplete data, but by using our sparse sketch, the running time
drops by a factor of five or more.
\end{abstract}

\def\norm#1{\|#1\|}
\newcommand{\remove}[1]{}
\def\transp{^T}
\def\trace{\text{trace}}
\def\opt{\textsc{opt}}
\section{Introduction}
\label{section:intro}
Principal components analysis constructs a low dimensional subspace of the
data such that projection of the data onto this subspace 
preserves as much information as possible (or equivalently maximizes 
the variance of the projected data).
The earliest reference to principal components analysis (PCA) is in
\cite{P1901}. Since then, PCA has evolved into a classic tool
for data analysis.
A challenge for the interpretation of the principal 
components (or factors)
 is that they can be linear combinations of \emph{all} the original 
variables.
When the original variables have direct physical 
significance (e.g. genes in biological applications or assets in
financial applications) it is desirable to have factors 
which have loadings on only a small
number of the original variables. 
These interpretable factors are 
\emph{sparse principal components (SPCA)}. 

The question we address is not how to better perform 
sparse PCA; rather, it is whether one can perform sparse PCA on 
\emph{incomplete
data}  and be assured some degree of 
success.
(Read: can one do sparse PCA when you have a 
small sample of data points and those data points have missing features?).
Incomplete data is a situation that one is confronted 
with all too often in machine learning. For example, with 
user-recommendation data, one does not have all the ratings of 
any given user. Or in a privacy preserving setting, a client
may not want to give you all entries in the data matrix.
In such a setting, our goal is to show that if the samples that you 
do get are chosen carefully, the sparse PCA 
features of the data can be recovered within some provable error bounds.
A significant part of this work is
to demonstrate our algorithms on a variety of data sets.

More formally,
The data matrix is \math{\matA\in\R^{m\times n}} (\math{m} data points in 
\math{n} dimensions).
Data matrices often have low effective rank.
Let \math{\matA_k} be the best rank-\math{k} approximation to 
\math{\matA}; in practice, it is often possible to 
choose a small value of \math{k} for which $\TsNorm{\matA-\matA_k}$
is small. The best rank-\math{k} approximation \math{\matA_k}
is obtained by projecting \math{\matA} onto the subspace spanned by its
top-\math{k} principal components \math{\matV_k}, which is
the \math{n\times k} matrix containing the top-\math{k} right singular
vectors of \math{\matA}. These top-\math{k} principal components are
the solution to the variance maximization problem:
\mand{
\matV_k=\mathop{\arg\max}\limits_{\matV\in\R^{n\times k},\matV\transp\matV=\matI}
\trace(\matV\transp\matA\transp\matA\matV).
}
We denote the maximum variance attainable by \math{\opt_k}, which is the
sum of squares of the top-\math{k} singular values
of \math{\matA}.
To get sparse principal components, 
you add a sparsity constraint to the optimization problem: every column of \math{\matV} should have at most 
\math{r} non-zero entries (the sparsity parameter \math{r} is an input),
\mld{
\matS_k=\mathop{\arg\max}\limits_{\matV\in\R^{n\times k},\matV\transp\matV=\matI,
\norm{\matV^{(i)}}_0\le r}
\trace(\matV\transp\matA\transp\matA\matV).\label{eq:spca}
}
The sparse PCA problem is itself a very hard problem that is not only
NP-hard, but also inapproximable~\citep{M2015}
There are many
heuristics for obtaining 
sparse factors~\citep{CJ95,TJU03,ZHT06,AEJL07,ABE08,MWA06a,SH08} 
including some approximation algorithms with provable guarantees
%\cite{APD2014,ourpaper}.
\cite{APD2014}. The existing research typically addresses the task of
getting just the top principal component (\math{k=1}). While the
sparse PCA problem is hard and interesting, it is \emph{not} the focus of this
work.

We address the question: 
What if you do not know \math{\matA}, but only have a sparse sampling of
some of the entries in \math{\matA} (incomplete data)?
The sparse sampling is used to construct
a \emph{sketch} of \math{\matA}, denoted 
\math{\tilde\matA}. There is not much else to do but solve the sparse 
PCA problem with the sketch \math{\tilde\matA} instead of 
the full data \math{\matA} to get \math{\tilde\matS_k},
\mld{
\tilde\matS_k=
\mathop{\arg\max}\limits_{\matV\in\R^{n\times k},\matV\transp\matV=\matI,\norm{\matV^{(i)}}_0\le r}
\trace(\matV\transp\tilde\matA\transp\tilde\matA\matV).
\label{eq:sketchspca}
}
We study how \math{\tilde\matS_k} performs as an approximation
to \math{\matS_k} with respective to the objective that we are trying to 
optimize, namely 
\math{\trace(\matS\transp\matA\transp\matA\matS)} 
--- the quality of approximation is
measured with respect to the true \math{\matA}. 
We show that 
the quality of approximation is controlled by how well 
\math{\tilde\matA\transp\tilde\matA} 
approximates \math{\matA\transp\matA} as measured by the spectral
norm of the deviation
\math{\matA\transp\matA-\tilde\matA\transp\tilde\matA}. 
This is a general result that does not
rely on how one constructs the sketch
\math{\tilde\matA}.
\begin{theorem}[Sparse PCA from a Sketch]\label{theorem:main1}
Let \math{\matS_k} be a solution to the sparse PCA problem that 
solves \r{eq:spca}, and \math{\tilde\matS_k} a solution to the sparse PCA 
problem for the sketch \math{\tilde\matA} which solves 
\r{eq:sketchspca}. Then,
\mand{
\trace(\tilde\matS_k\transp\matA\transp\matA\tilde\matS_k)\ge
\trace(\matS_k\transp\matA\transp\matA\matS_k)-2k\norm{\matA\transp\matA-
\tilde\matA\transp\tilde\matA}_2.
}
\end{theorem}
Theorem~\ref{theorem:main1} says that if we can closely approximate \math{\matA}
with \math{\tilde\matA}, then we can compute, from \math{\tilde\matA},
sparse components which capture almost as much variance as the optimal 
sparse components computed from the full data \math{\matA}.

In our setting, the sketch \math{\tilde\matA} is
 computed from 
a sparse sampling of the data elements in \math{\matA} (incomplete data).
To determine which elements to sample, and how to form the sketch,
we leverage some
recent results in elementwise matrix completion~(\cite{KDM15}).
In a nutshell, if one
samples larger data elements with higher probability than
smaller data elements,
then, for the resulting sketch \math{\tilde\matA}, the error 
\math{\norm{\matA\transp\matA-\tilde\matA\transp\tilde\matA}_2} will be small. 
The details of the sampling
scheme and how the error depends on the number of samples is
given in Section~\ref{sec:tool}. Combining the bound on 
\math{\norm{\matA-\tilde\matA}_2} from Theorem~\ref{thm:element_sampling} in  
Section~\ref{sec:tool} with Theorem~\ref{theorem:main1}, we get 
our
main result:
\begin{theorem}[Sampling Complexity for Sparse PCA]\label{theorem:main2}
Sample \math{s} data-elements from \math{\matA\in\R^{m\times n}} to form the 
sparse sketch \math{\tilde\matA} using Algorithm~\ref{alg:proto_alg1}.
Let \math{\matS_k} be a solution to the sparse PCA problem that 
solves \r{eq:spca}, and let \math{\tilde\matS_k}, which solves 
\r{eq:sketchspca}, be a solution to the sparse PCA 
problem for the sketch \math{\tilde\matA} formed from the \math{s}
sampled data elements. Suppose the number of samples \math{s} satisfies 
\mand{
s\ge\frac{2k^2}{\epsilon^2}\left(\rho^2+\frac{\epsilon\gamma}{3k}\right)
\log\left(
\frac{m+n}{\delta}\right)
}
(\math{\rho^2} and \math{\gamma} are dimensionless quantities
that depend only on \math{\matA}).
Then, with probability at least \math{1-\delta}
\mand{
\trace(\tilde\matS_k\transp\matA\transp\matA\tilde\matS_k)\ge
\trace(\matS_k\transp\matA\transp\matA\matS_k)-\epsilon(2+\epsilon/k)
\norm{\matA}_2^2.
}
\end{theorem}
The dependence of \math{\rho^2} and \math{\gamma} on \math{\matA} are given
in Section~\ref{sec:tool}. Roughly speaking, we can ignore the term with
\math{\gamma} since it is multiplied by \math{\epsilon/k}, and 
\math{\rho^2=O(\tilde k\max\{m,n\})}, where \math{\tilde k} is the
stable (numerical) rank of \math{\matA}.
To paraphrase Theorem~\ref{theorem:main2}, 
when the stable rank is a small constant,
with \math{O(k^2\max\{m,n\})}
samples, one can recover almost as good sparse principal components
as with all data (the price being a small fraction of the optimal
variance, since \math{\opt_k\ge\norm{\matA}_2^2}). 
As far as we know, this is the first result to show 
that it is possible to provably recover sparse PCA from incomplete
data. We also give an application of 
Theorem~\ref{theorem:main1} to running sparse PCA after 
``denoising'' the  data
 using a greedy thresholding algorithm that sets the small elements to 
zero (see Theorem~\ref{theorem:thresholding}). Such denoising 
is appropriate when the observed matrix has been 
element-wise perturbed by small noise, and the uncontaminated data
matrix is sparse and contains
large elements. We show that if an appropriate fraction of the
(noisy) data is set to zero, one can still recover sparse principal
components. 
%We use this result to show under certain settings
%that one can recover better sparse components 
%from corrupted data after setting small elements
%to zero. 
This gives a principled approach to regularizing sparse
PCA in the presence of small noise when the data is sparse.

Not only do our algorithms preserve the quality of the
sparse principal components,
but iterative algorithms for sparse PCA, whose running time is proportional
to the number of non-zero entries in the input matrix, benefit from the
sparsity of \math{\tilde\matA}. Our experiments show about
five-fold speed gains while producing near-comparable sparse components 
using less than
10\% of the data.
%{\bf XXX Can we prove a theorem to this effect for some standard heuristic
%solver (for example iterative hard thresholding)}

\paragraph{Discussion.}
In summary, we show that one can recover sparse PCA from incomplete
data while gaining computationally at the same time.
Our result holds for the optimal sparse components from
\math{\matA}
versus from 
\math{\tilde\matA}. One cannot efficiently
find
these optimal components (since the problem is NP-hard to even approximate),
so one runs a heuristic, in which case the approximation error of the 
heuristic would have to be taken into account. Our experiments
show that using the incomplete data with the heuristics is 
just as good as those same heuristics with the complete data.

In practice, one may not be able to sample the data, but rather
the samples are given to you. Our result establishes that if the
samples are chosen with larger values being more likely, then
one can recover sparse PCA. In practice one
has no choice but to run the sparse PCA on these sampled elements
and hope. Our theoretical results suggest that the outcome will be
reasonable. This is because, while we 
do not have specific control over what samples we get, the 
samples are likely to represent the larger elements. 
For example, with user-product recommendation data,
users are more likely to rate items they either really like 
(large positive value) or really
dislike (large negative value).

%
%\subsection{Notation}\label{sec:notation}
\paragraph{Notation.}\label{sec:notation}
We use bold uppercase (e.g., $\matX$) for
matrices and bold lowercase (e.g., $\x$) for column vectors. 
The $i$-th row of $\matX$ is $\matX_{(i)}$, and the $i$-th column of $\matX$ 
is $\matX^{(i)}$. Let $[n]$ denote the set $\{1, 2, ..., n\}$. 
$\mathbb{E}(X)$ is the expectation of a random variable $X$; 
for a matrix, $\mathbb{E}(\matX)$ denotes the element-wise expectation. For a matrix $\matX \in \mathbb{R}^{m \times n}$, the Frobenius norm $\FNorm{\matX}$ is 
$\FNormS{\matX} = \sum_{i,j=1}^{m,n} \matX_{ij}^2$, and the spectral (operator)
 norm $\TNorm{\matX}$ is 
$\TNorm{\matX} = \text{max}_{\TNorm{\y}=1}\TNorm{\matX\y}$. 
We also have the \math{\ell_1} and \math{\ell_0} norms:
$\ONorm{\matX} = \sum_{i,j=1}^{m,n} \abs{\matX_{ij}}$ and $\ZNorm{\matX}$ 
(the number of non-zero entries in $\matX$). The $k$-th largest singular value of $\matX$ is $\sigma_k(\matX)$.
and $\log x$ is the natural logarithm of $x$. 

\section{Sparse PCA from a Sketch}
\label{section:sketch}

In this section, we will prove Theorem~\ref{theorem:main1} and give a 
simple application to zeroing small fluctuations as a way to 
regularize to noise. In the next section we will use a more sophisticated
way to select the elements of the matrix allowing us to tolerate
a sparser matrix (more incomplete data) but still recovering
sparse PCA to reasonable accuracy.

Theorem~\ref{theorem:main1} will be a corollary of
a more general result, for a class of optimization problems involving a 
Lipschitz objective function over an arbitrary (not necessarily convex)
domain. Let \math{f(\matV,\matX)} be a function that is defined for 
a matrix variable \math{\matV} and a matrix parameter \math{\matX}. 
The optimization variable \math{\matV} is in some feasible set
\math{\cl{S}} which is arbitrary. The parameter \math{\matX} is also
arbitrary. We assume that \math{f} is Lipschitz in \math{\matX} with 
Lipschitz constant \math{\gamma}. So,
\mand{
|f(\matV,\matX)-f(\matV,\tilde\matX)|\le\gamma(\matX)\norm{\matX-\tilde\matX}_2
\qquad\forall\matV\in\cl{S}.
}
(Note we allow the Lipschitz constant to depend on \math{\matX} but not
\math{\matV}.)
The next lemma is the key tool we need to prove Theorem~\ref{theorem:main1}
and it may be on independent interest in other optimization settings. We are 
interested in maximizing \math{f(\matV,\matX)} w.r.t. \math{\matV} to obtain
\math{\matV^*}. But, we only
have an approximation \math{\tilde\matX} for  \math{\matX}, and so we maximize
\math{f(\matV,\tilde\matX)} to obtain \math{\tilde\matV^*}, which will be 
a suboptimal solution with respect to \math{\matX}. We wish to bound
\math{f(\matV^*,\matX)-f(\tilde\matV^*,\matX)} which quantifies how 
suboptimal \math{\tilde\matV^*} is w.r.t. \math{\matX}. 
\begin{lemma}[Surrogate optimization bound]\label{lemma:tool1}
Let \math{f(\matV,\matX)} be \math{\gamma}-Lipschitz w.r.t.
\math{\matX} over the domain \math{\matV\in\cl{S}}.
Define
\eqan{
\matV^*=\arg\max_{\matV\in\cl{S}}f(\matV,\matX); \qquad 
\tilde\matV^*=\arg\max_{\matV\in\cl{S}}f(\matV,\tilde\matX).
%\matV^*&=&\arg\max_{\matV\in\cl{S}}f(\matV,\matX);\\
%\tilde\matV^*&=&\arg\max_{\matV\in\cl{S}}f(\matV,\tilde\matX).
}
Then, 
$$
f(\matV^*,\matX)-f(\tilde\matV^*,\matX)\le 2\gamma(\matX)\norm{\matX-\tilde\matX}_2.
$$
%\mand{
%f(\matV^*,\matX)-f(\tilde\matV^*,\matX)\le 2\gamma(\matX)\norm{\matX-\tilde\matX}_2.
%}
\end{lemma}
In the lemma, the function \math{f} and the domain \math{\cl{S}} are arbitrary.
In our setting, \math{\matX\in\R^{n\times n}},
the domain \math{\cl{S}=\{\matV\in\R^{n\times k};\matV\transp\matV=\matI_k;
\norm{\matV^{(j)}}_0\le r\}}, and 
\math{f(\matV,\matX)=\trace(\matV\transp\matX\matV)}.
We first show that 
\math{f} is Lipschitz w.r.t. \math{\matX} with \math{\gamma=k} (a constant
independent of \math{\matX}). Let the representation of
\math{\matV} by its columns be
\math{\matV=[\vv_1,\ldots,\vv_k]}.
%and let $\phi = |\trace(\matV\transp\matX\matV)-\trace(\matV\transp\tilde\matX\matV)|$. 
Then,
\eqan{
|\trace(\matV\transp\matX\matV)-\trace(\matV\transp\tilde\matX\matV)| 
%&=& |\trace(\matV\transp\matX\matV)-\trace(\matV\transp\tilde\matX\matV)|\\
=
%|\trace(\matV\transp(\matX-\tilde\matX)\matV)|= 
|\trace((\matX-\tilde\matX)\matV\matV\transp)|
%
%=\left|\sum_{i=1}^k\vv_i\transp(\matX-\tilde\matX)\vv_i\right|\\
\le \sum_{i=1}^k\sigma_i(\matX-\tilde\matX)
%\le\sum_{i=1}^k\norm{\matX-\tilde\matX}_2\vv_i\transp\vv_i\\
%&\le&\sum_{i=1}^k|\vv_i\transp(\matX-\tilde\matX)\vv_i|
%\le\sum_{i=1}^k\norm{\matX-\tilde\matX}_2\vv_i\transp\vv_i\\
\le k\norm{\matX-\tilde\matX}_2
}
where, $\sigma_i(\matA)$ is the $i$-th largest singular value of $\matA$ 
(we used Von-neumann's trace inequality and the fact that
\math{\matV\matV\transp} is a \math{k}-dimensional projection). 
Now, by Lemma~\ref{lemma:tool1},
\mand{
\trace({\matV^*}\transp\matX\matV^*)-
\trace(\tilde\matV^{*T}\matX\tilde\matV^*)
\le 2k\norm{\matX-\tilde\matX}_2.
}
Theorem~\ref{theorem:main1} follows by setting 
\math{\matX=\matA\transp\matA} and 
\math{\tilde\matX=\tilde\matA\transp\tilde\matA}.

\paragraph{Greedy thresholding.}
We give the simplest scenario of incomplete data where 
Theorem~\ref{theorem:main1} gives some reassurance that one can compute 
good
sparse principal components. Suppose the smallest data elements
have been set to zero. This can happen, for example, if only the
largest elements are measured, or in a noisy setting if the
small elements are treated as noise and set to zero. So 
\eqan{
\tilde\matA_{ij}=
\begin{cases}
\matA_{ij}&|\matA_{ij}|\ge\delta;\\
0&|\matA_{ij}|<\delta.
\end{cases}
}
Recall \math{\tilde k=\norm{\matA}^2_F/\norm{\matA}^2_2} (stable rank of 
\math{\matA}), and define 
%the truncated Frobenius norm of \math{\matA} as 
\math{\norm{\matA_\delta}^2_F=\sum_{|\matA_{ij}|<\delta}\matA_{ij}^2}.
Let \math{\matA=\tilde\matA+\Delta}. By construction,
\math{\norm{\Delta}_F^2=\norm{\matA_\delta}^2_F}. Then,
\begin{eqnarray}\label{eq:quad-bound}
\norm{\matA\transp\matA-\tilde\matA\transp\tilde\matA}_2
=
\norm{\matA\transp\Delta+\Delta\transp\matA-\Delta\transp\Delta}_2
\le
2\norm{\matA}_2\norm{\Delta}_2+\norm{\Delta}_2^2.
\end{eqnarray}
Suppose the zeroing of elements only loses a fraction of 
the energy in \math{\matA}, i.e. \math{\delta} is selected
so that \math{\norm{\matA_\delta}^2_F\le\epsilon^2\norm{\matA}^2_F/\tilde k};
that is an \math{\epsilon/\tilde k} fraction of the 
total variance in \math{\matA} has been lost in the 
unmeasured (or zero) data. Then 
\mand{
\norm{\Delta}_2
\le\norm{\Delta}_F
\le\frac{\epsilon}{\sqrt{\tilde k}}\norm{\matA}_F
=\epsilon\norm{\matA}_2.}
%and we have the following theorem.
\begin{theorem}\label{theorem:thresholding}
Suppose that \math{\tilde\matA} is created from \math{\matA} by zeroing
all elements that are less than \math{\delta}, and 
\math{\delta} is such that the truncated norm satisfies
\math{\norm{\matA_\delta}_2^2\le\epsilon^2\norm{\matA}_F^2/\tilde k}. Then
the sparse PCA solution \math{\tilde\matV^*} satisfies
\mand{
\trace(\tilde\matV^{*T}\matA\matA\tilde\matV^*)\ge
\trace({\matV^*}\transp\matA\matA\transp\matV^*)
-2k\epsilon\norm{\matA}_2^2(2+\epsilon).
}
\end{theorem}
Theorem~\ref{theorem:thresholding} shows that it is possible to 
recover sparse PCA after setting small elements to zero.
This is appropriate when most of 
the elements in \math{\matA} are small noise and a few of the elements
in \math{A} contain large data elements. For example if your 
data consists of sparse \math{O(\sqrt{nm})} large elements 
(of magnitude, say, 1) and
many \math{nm-O(\sqrt{nm})} small elements whose magnitude is
\math{o(1/\sqrt{nm})} (high signal-to-noise setting), then
\math{\norm{\matA_\delta}_2^2/ \norm{\matA}_2^2\rightarrow0} and with 
just a sparse sampling of the \math{O(\sqrt{nm})}
large elements (very incomplete data), one recovers
near optimal sparse PCA.

Greedily keeping only the large elements of the matrix 
requires a particular structure in \math{\matA} to work,
and it is based on a crude Frobenius-norm bound for the spectral
error. In Section~\ref{sec:tool}, we 
use recent results in element-wise matrix sparsification
to
choose the elements in a randomized
way, with a bias toward large elements.
With high probability, one can directly bound the 
spectral error and hence get better performance. But first, let us
prove Lemma~\ref{lemma:tool1}

%\subsection{A Proof of Lemma~\ref{lemma:tool1}}
\paragraph{A Proof of Lemma~\ref{lemma:tool1}.}
We need the following 
lemma.
\begin{lemma}\label{lemma:basic1}
Let \math{f} and \math{g}
be functions on a domain
\math{\cl{S}}. Then,
\mand{
\sup_{x\in\cl{S}}f(x)-\sup_{y\in\cl{S}}g(y)
\le
\sup_{x\in\cl{S}}(f(x)-g(x)).
}
\end{lemma}
\begin{proof} 
$$
\sup_{x\in\cl{S}}(f(x)-g(x))
\ge
{f(x)-g(x)}\ge
{f(x)-\sup_{y\in\cl{S}}g(y)},
\quad \forall x\in\cl{S}.
$$
Since the RHS holds for all \math{x}, it follows that
\math{\sup_{x\in\cl{S}}(f(x)-g(x))} is an upper bound for 
\math{f(x)-\sup_{y\in\cl{S}}g(\matU)}, and hence
\eqan{
\sup_{x\in\cl{S}}(f(x)-g(x))
&\ge&
\sup_{x\in\cl{S}}\left(f(x)-\sup_{y\in\cl{S}}g(y)\right)\\
&=&
\sup_{x\in\cl{S}}f(x)-\sup_{y\in\cl{S}}g(y).
}
\end{proof}

\begin{proof}(Lemma~\ref{lemma:tool1})
Suppose that \math{\max_{\matV\in\cl{S}}f(\matV,\matX)} is attained at 
\math{\matV^*} and 
\math{\max_{\matV\in\cl{S}}f(\matV,\tilde\matX)} is attained at 
\math{\tilde\matV^*}, and define
\math{\epsilon=f(\matV^*,\matX)-f(\tilde\matV^*,\matX)}.
We have that
\eqan{
\epsilon
&=&
f(\matV^*,\matX)-f(\tilde\matV^*,\tilde\matX)
+f(\tilde\matV^*,\tilde\matX)-f(\tilde\matV^*,\matX)\\
&=&
\max_\matV f(\matV,\matX)-\max_\matU f(\matU,\tilde\matX)
+f(\tilde\matV^*,\tilde\matX)-f(\tilde\matV^*,\matX)\\
&\le&
\max_\matV \left(f(\matV,\matX)-f(\matV,\tilde\matX)\right)
+f(\tilde\matV^*,\tilde\matX)-f(\tilde\matV^*,\matX),\\
}
where the last step follows from Lemma~\ref{lemma:basic1}. Therefore,
\eqan{
|\epsilon|
&\le&
\max_\matV \left|f(\matV,\matX)-f(\matV,\tilde\matX)\right|
+|f(\tilde\matV^*,\tilde\matX)-f(\tilde\matV^*,\matX)|\\
&\le&
\max_\matV \gamma(\matX)\norm{\matX-\tilde\matX}_2+
\gamma(\matX)\norm{\matX-\tilde\matX}_2\\
&=&
2\gamma(\matX)\norm{\matX-\tilde\matX}_2.}
(We used the
Lipschitz condition in the second step.)
\end{proof}

\subsection{An $(\ell_1,\ell_2)$-Sampling Based Sketch} \label{sec:tool}

In the previous section, we created the sketch 
by deterministically
setting the small data elements to zero. Instead, we could 
randomly select the data elements
to keep. It is natural to bias this random sampling toward the larger
elements. Therefore, we define sampling probabilities for each data
element \math{\matA_{ij}} which are 
proportional to a mixture of the absolute value and square of the data element:
\begin{eqnarray}\label{eqn:hybrid_prob}
p_{ij} = \alpha  \frac{|\matA_{ij}|}{\ONorm{\matA}} + (1-\alpha)\frac{\matA_{ij}^2}{\FNormS{\matA}},
\end{eqnarray}
where \math{\alpha \in (0,1]} is a mixing parameter.
Such a sampling probability was used in \cite{KDM15} to sample
data elements in independent trials to get a sketch \math{\tilde\matA}.
We repeat the 
prototypical algorithm for element-wise matrix sampling
in Algorithm \ref{alg:proto_alg1}. 
\begin{algorithm}[t]
\caption{Hybrid $(\ell_1,\ell_2)$-Element Sampling\label{alg:proto_alg1}}
\textbf{Input:} $\matA\in \R^{m \times n}$; \#\,samples $s$; probabilities \math{\{p_{ij}\}}.

\begin{algorithmic}[1]
\STATE Set $\tilde{\matA} = \textbf{0}_{m \times n}$.
\FOR{$t = 1\ldots s$ (i.i.d. trials with replacement)}
\STATE Randomly sample indices $(i_t, j_t) \in [m] \times [n]$ with
%
%\centerline{\math{\mathbb{P}\left[ (i_t, j_t) =  (i,j)\right]=p_{ij}.}}
{\math{\mathbb{P}\left[ (i_t, j_t) =  (i,j)\right]=p_{ij}.}}
\STATE Update \math{\tilde\matA}: \math{\displaystyle\tilde{\matA}_{ij}\gets
\tilde{\matA}_{ij} + \frac{\matA_{ij}}{s\cdot p_{ij}}.}
\ENDFOR
\RETURN $\tilde{\matA}$ (with at most \math{s} non-zero entries).
\end{algorithmic}
\end{algorithm} 

\noindent
Note that unlike with the deterministic zeroing of small elements, in this
sampling scheme, one samples the element \math{\matA_{ij}} with
probability \math{p_{ij}} and then \emph{rescales} it by \math{1/p_{ij}}.
To see the intuition for this rescaling, consider the expected
outcome for a single sample:
\mand{
\Exp[\tilde\matA_{ij}]=p_{ij}\cdot(\matA_{ij}/p_{ij})+(1-p_{ij})\cdot 0=\matA_{ij};
}
that is, \math{\tilde\matA} is a sparse but \emph{unbiased} estimate
for \math{\matA}.
This unbiasedness
holds for any choice of the sampling probabilities
$p_{ij}$ defined over the elements of $\matA$ in Algorithm \ref{alg:proto_alg1}.
However, for an appropriate choice of the 
sampling probabilities, we get much more than unbiasedness;
we can control the spectral norm of the deviation, 
\math{\norm{\matA-\tilde\matA}_2}.
In particular, 
the hybrid-$(\ell_1,\ell_2)$ distribution in \r{eqn:hybrid_prob} 
was analyzed in \cite{KDM15}, where they suggest an optimal 
choice for the mixing parameter \math{\alpha^*} which 
minimizes the theoretical bound on \math{\norm{\matA-\tilde\matA}_2}.
This algorithm to choose \math{\alpha^*} is summarized in Algorithm~\ref{alg:alp*}.

Using the probabilities in (\ref{eqn:hybrid_prob}) to create the 
sketch \math{\tilde\matA} using Algorithm~\ref{alg:proto_alg1},
with \math{\alpha^*} selected using Algorithm~\ref{alg:alp*}, one can prove 
a bound for \math{\norm{\matA-\tilde\matA}_2}. We state a simplified version
of the bound from \cite{KDM15}
in Theorem \ref{thm:element_sampling}.
\begin{theorem}[\cite{KDM15}]\label{thm:element_sampling}
Let $\matA \in \mathbb{R}^{m \times n}$ and  let $\epsilon > 0$ be 
an accuracy parameter. Define probabilities
$p_{ij}$ as in (\ref{eqn:hybrid_prob}) with \math{\alpha^*} chosen using 
Algorithm~\ref{alg:alp*}.
Let $\tilde{\matA}$ be the sparse sketch produced using
Algorithm~\ref{alg:proto_alg1} with a number of samples 
\mand{
s \ge \frac{2}{\epsilon^2}\left( \rho^2 + \gamma\epsilon/3\right) \log\left(\frac{m+n}{\delta}\right),
}
where 
\mand{
\rho^2=
\frac{\tilde{k}\cdot \max\{m,n\}}{{\alpha \cdot  \tilde{k}}\cdot \frac{\TNorm{\matA}}{\ONorm{\matA}}+ (1-\alpha)}, \quad\text{and}\quad  
\gamma \leq 1 + \frac{\sqrt{mn\tilde{k}}}{\alpha}.
}
Then, with probability at least $1 - \delta$,
$$\TsNorm{\matA - \tilde{\matA}} \leq \epsilon \TNorm{\matA}.$$
\end{theorem}
\remove{
%============ proof =============
\begin{proof}
\begin{eqnarray*}
\frac{\alpha \cdot  \FNormS{\matA}}{\abs{\matA_{ij}}\cdot \ONorm{\matA}}+ (1-\alpha) 
&\geq&
 \frac{\alpha \cdot  \FNormS{\matA}}{\TNorm{\matA}\cdot \ONorm{\matA}}+ (1-\alpha)\\
&=& {\alpha \cdot \tilde{k}}\cdot \frac{\TNorm{\matA}}{\ONorm{\matA}}+ (1-\alpha)
\end{eqnarray*}
$$
\frac{\xi_{ij}}{\TNormS{\matA}} \leq \frac{\tilde{k}}{{\alpha \cdot  \tilde{k}}\cdot \frac{\TNorm{\matA}}{\ONorm{\matA}}+ (1-\alpha)}
$$
$$
\frac{\rho^2}{\TNormS{\matA}} \leq \frac{\tilde{k}\cdot \max\{m,n\}}{{\alpha \cdot  \tilde{k}}\cdot \frac{\TNorm{\matA}}{\ONorm{\matA}}+ (1-\alpha)} 
$$
$$
\frac{\gamma}{\TNorm{\matA}} \leq 1+ \frac{\ONorm{\matA}}{\alpha \cdot \TNorm{\matA}}
\leq 1 + \frac{\sqrt{mn\tilde{k}}}{\alpha}
$$
For $\alpha=1$,
$$
\frac{\rho^2}{\TNormS{\matA}} \leq \max\{m,n\}\frac{\ONorm{\matA}}{\TNorm{\matA}}
$$
\\
\end{proof}
}
\begin{proof} Follows from the bound in \cite{KDM15}.
\end{proof}

Recall that \math{\tilde k} is the stable rank of \math{\matA}. 
In practice, \math{\alpha^*} is bounded away from 0 and 1, and so 
 \math{s=O(\epsilon^{-2}\tilde k\max\{m,n\})} samples suffices to get 
a sketch \math{\tilde\matA} for which 
\math{\norm{\matA-\tilde\matA}_2\le\epsilon\norm{\matA}}.
This is exactly what we need to prove Theorem~\ref{theorem:main2}.

%\subsection{Proof of Theorem~\ref{theorem:main2}}
\paragraph{Proof of Theorem~\ref{theorem:main2}.}
The number of samples \math{s} in Theorem~\ref{theorem:main2} corresponds
to the number of samples needed in Theorem~\ref{thm:element_sampling}
with the error tolerance \math{\epsilon/k}.
Using \r{eq:quad-bound} (where \math{\Delta=\matA-\tilde\matA}) and
Theorem~\ref{thm:element_sampling}, we have that
\eqar{
\norm{\matA\transp\matA-\tilde\matA\transp\tilde\matA}_2
%&\le&
%2\norm{\matA}_2\norm{\Delta}_2+\norm{\Delta}_2^2\nonumber\\
&\le&
\frac{2\epsilon}{k}\norm{\matA}_2^2+\frac{\epsilon^2}{k^2}\norm{\matA}_2^2.
\label{eq:proof:main2-1}
}
Using \r{eq:proof:main2-1} in Theorem~\ref{theorem:main1} gives
Theorem~\ref{theorem:main2}.

\begin{algorithm}[t]
\caption{Optimal Mixing Parameter \math{\alpha^*}\label{alg:alp*}}
\textbf{Input:} $\matA\in \R^{m \times n}$.
\begin{algorithmic}[1]
\STATE Define two functions of \math{\alpha} that depend on \math{\matA}:
\eqan{
\rho^2(\alpha)= 
\max\left\{\max_i \sum_{j=1}^{n}\xi_{ij},\max_j \sum_{i=1}^{m}\xi_{ij}\right\} - \sigma_{min}^2(\matA);\\
\gamma(\alpha)=
\max_{\stackrel{i,j:}{\matA_{ij}\neq 0}}\left\{\frac{\ONorm{\matA}}{\alpha + (1-\alpha)\frac{\ONorm{\matA}\cdot \abs{\matA_{ij}}}{\FNormS{\matA}}}\right\} + \TNorm{\matA};
}
where,
$$
\xi_{ij} = {\FNormS{\matA}}/\left (\frac{\alpha \cdot  \FNormS{\matA}}{\abs{\matA_{ij}}\cdot \ONorm{\matA}}+ (1-\alpha)\right ),
\text{ for } \matA_{ij}\neq 0.
$$
\STATE Find \math{\alpha^*\in (0,1]} to minimize
\math{\displaystyle
\rho^2(\alpha) + \gamma(\alpha)\epsilon\TNorm{\matA}/3.}
\RETURN $\alpha^*$
\end{algorithmic}
\end{algorithm} 

\section{Experiments}\label{experiments}

We show the experimental performance of sparse PCA from a sketch
using several real data matrices. As we mentioned, sparse PCA
is NP-Hard, and so we must use heuristics. These heuristics are discussed
next, followed by the data, the experimental design and finaly the results.

\subsection{Algorithms for Sparse PCA}

Let $\mathcal{G}$ (ground truth) denote the algorithm which computes
the principal components (which may not be sparse) of the full
data
matrix \math{\matA}; the optimal variance is \math{\opt_k}. 
We consider six heuristics for getting sparce principal components.

\begin{tabular}{p{0.035\textwidth}p{\textwidth}}
$\mathcal{G}_{\text{max}, r}$ &\qquad The $r$ largest-magnitude entries in each 
principal component generated by \math{\mathcal{G}}.\\
$\mathcal{G}_{\text{sp}, r}$ & \qquad 
\math{r}-sparse components using the \textit{Spasm} toolbox of \cite{KCLE12}
with \math{\matA}.\\
$\mathcal{H}_{\text{max}, r}$& \qquad
The $r$ largest entries of the 
principal components for the \math{(\ell_1,\ell_2)}-sampled
sketch \math{\tilde\matA}.\\
$\mathcal{H}_{\text{sp}, r}$& \qquad
\math{r}-sparse components using \textit{Spasm} with 
the \math{(\ell_1,\ell_2)}-sampled
sketch \math{\tilde\matA}.\\
$\mathcal{U}_{\text{max}, r}$& \qquad
The $r$ largest entries of the 
principal components for the \emph{uniformly} sampled
sketch \math{\tilde\matA}.\\
$\mathcal{U}_{\text{sp}, r}$& \qquad
\math{r}-sparse components using \textit{Spasm} with 
the uniformly sampled
sketch \math{\tilde\matA}.\\
\end{tabular}
\remove{
Table \ref{fig:sparse_PC} summarizes the key properties of these algorithms.

\begin{table}[t]
\begin{center}
\scalebox{1}{
    \begin{tabular}{| c || c | c | c | c | c | c |}
   \hline
& $\mathcal{G}_{\text{max}, r}$& $\mathcal{G}_{\text{sp}, r}$ & $\mathcal{H}_{\text{max}, r}$ 
& $\mathcal{H}_{\text{sp}, r}$  & $\mathcal{U}_{\text{max}, r}$ & $\mathcal{U}_{\text{sp}, r}$  \\ \hline \hline
Algo & naive & spasm & naive & spasm & naive & spasm \\ \hline
Data & full & full & sparse & sparse & sparse & sparse \\ \hline
$p_{ij}$ & - & - & $(\ell_1,\ell_2)$ & $(\ell_1,\ell_2)$ & unif & unif \\ \hline
   \end{tabular}
}
\caption
{
Summary of various SPCA algorithms.
}
\label{fig:sparse_PC}
\end{center}
\end{table}
%  ===============================================
}
\\
The outputs of an algorithm \math{\cl{Z}} are 
sparse principal components $\matV$, and  the metric we are interested
in is the variance, 
$f(\cl{Z})= \trace(\matV^T\matA^T\matA\matV)$, where $\matA$ is the original 
centered data. We consider the following statistics.\\

\begin{tabular}{l| p{13cm}}
%$\displaystyle\frac{f(\mathcal{G}_{\text{sp}, r})}{f(\mathcal{G})}$& 
%Relative loss due to the sparsity constraint $r$; 
%our sketch based algorithms \emph{do not} address this loss.\\
$\displaystyle
\frac{f(\mathcal{G}_{\text{max}, r})}{f(\mathcal{G}_{\text{sp},r})}$&  
Relative loss of greedy thresholding 
versus \emph{Spasm}, illustrating the value of a good
sparse PCA algorithm.
Our sketch based algorithms \emph{do not} 
address this loss. \\
&\\ 
$\displaystyle
\frac{f(\mathcal{H}_{\text{max/sp}, r})}{f(\mathcal{G}_{\text{max/sp},r})}$&  
Relative loss of using the \math{(\ell_1,\ell_2)}-sketch 
\math{\tilde\matA} instead of complete data \math{\matA}. A ratio close
to 1 is desired.\\
&\\
$\displaystyle
\frac{f(\mathcal{U}_{\text{max/sp}, r})}{f(\mathcal{G}_{\text{max/sp},r})}$&  
Relative loss of using the uniform sketch 
\math{\tilde\matA} instead of complete data \math{\matA}.
A benchmark to highlight the value of a good sketch.\\
&\\
\end{tabular} 

We also report on the computation time 
for the algorithms. We show results to confirm that
sparse PCA algorithms using the 
\math{(\ell_1,\ell_2)}-sketch are nearly comparable to those
same algorithms on the complete data; and, computing from a 
sparse sketch has a running time that is reduced
 proportionately to the sparsity.

\subsection{Data Sets}

We show results on 
image, text, stock, and gene expression data. 
We  briefly describe the datasets below.\\

\textbf{Digit Data ($m = 2313,\ n=256$):}
We use the \cite{Hull94} handwritten
zip-code
digit images (300 pixels/inch in 8-bit gray scale).
Each pixel is a feature (normalized to be in \math{[-1,1]}). 
Each $16 \times 16$ digit image forms a row of the
data matrix \math{\matA}.
We focus on three digits: 
``6'' (664 samples), ``9'' (644 samples), and ``1'' (1005 samples). \\

\textbf{TechTC Data ($m =139,\ n=15170$):}
We use the Technion Repository of Text Categorization Dataset (TechTC, 
see~\cite{GM04}) from the Open Directory Project (ODP).
Each
documents is represented as a probability distribution over
a bag-of-words, with words being the 
features -- we removed words with fewer than 5 letters. 
Each of the 139 documents forms a row in the data.\\

\textbf{Stock Data ($m =7056,\ n=1218$):}
We use S\&P100 stock market data of prices for 
1218 stocks collected between 1983 and 2011. 
This temporal dataset has 7056 snapshots of stock prices. 
The prices of each day form a row of the data matrix and a principal component
represents an ``index'' of sorts -- each stock is a feature.\\

\textbf{Gene Expression Data ($m =107,\ n=22215$):}
We use GSE10072 gene expression data for lung cancer from the 
NCBI Gene Expression Omnibus database. There are 107 samples 
(58 lung tumor cases and 49 normal lung controls) forming the rows
of the data matrix, with 
22,215 probes (features) from the GPL96 platform annotation table. 

\subsection{Results}\label{experimental_results}

We report results for primarily
the top principal component (\math{k=1}) which is the
case most considered in the literature. When \math{k>1}, 
our results do not qualitatively change.

%\subsubsection{Handwritten Digits}\label{sec:digit}
\paragraph{Handwritten Digits.}\label{sec:digit}
Using Algorithm~\ref{alg:alp*}, the optimal mixing parameter is
$\alpha^* = 0.42$. We sample approximately 
$7\%$ of the elements from the centered data using  
 \math{(\ell_1,\ell_2)}-sampling,
as well as uniform sampling. 
The performance for small of \math{r} is shown in
Table~\ref{fig:digit_obj_ratio}, including the running time \math{\tau}.

\begin{table}[!h]
\begin{center}
\begin{tabular}{c|cc|cc}
$r$ & 
$\displaystyle\frac{f(\mathcal{H}_{\text{max/\bf sp}, r})}{f(\mathcal{G}_{\text{max/\bf sp}, r})}$&
\math{\displaystyle\frac{\tau(\cl{G})}{\tau(\cl{H})}}&
$\displaystyle\frac{f(\mathcal{U}_{\text{max/\bf sp}, r})}{f(\mathcal{G}_{\text{max/\bf sp}, r})}$&   
\math{\displaystyle\frac{\tau(\cl{G})}{\tau(\cl{U})}}\\[2pt]
&&&&\\[-9pt]\hline &&&&\\[-9pt]
20& 1.01/\bf 0.89&6.03&1.13/\bf 0.56&4.7\\
40& 0.99/\bf 0.90&6.21&1.01/\bf 0.70&5.33\\
60& 0.99/\bf 0.98&5.96&0.97/\bf 0.80&5.33\\
80& 0.99/\bf 0.95&6.03&0.94/\bf 0.81&5.18\\
100& 0.99/\bf 0.98&6.22&0.95/\bf 0.87&5.08\\
   \end{tabular}
\caption{
[Digits] Comparison of sparse principal components from the 
\math{(\ell_1,\ell_2)}-sketch and uniform sketch.
}
\label{fig:digit_obj_ratio}
\end{center}
\end{table}

For this data, 
\math{f(\mathcal{G}_{\text{max}, r})/f(\mathcal{G}_{\text{sp}, r})\approx 0.23} ($r=10$),
so it is important to use a good sparse PCA algorithm.
We see from Table~\ref{fig:digit_obj_ratio} that 
the \math{(\ell_1,\ell_2)}-sketch significantly outperforms the
uniform sketch. A more extensive comparison of recovered variance is given in
Figure~\ref{fig:ErrorAll}(a). We also observe a speed-up of a factor of about
6 for the \math{(\ell_1,\ell_2)}-sketch.
We point out that the uniform sketch is reasonable for the digits data 
because most data elements are close to either  $+1$ or $-1$, since the
pixels are either black or white. 
\begin{figure}[!h]
\centering
\begin{subfigure}[b]{0.22\textwidth}
\includegraphics[height=3.6cm,width=3.8cm]{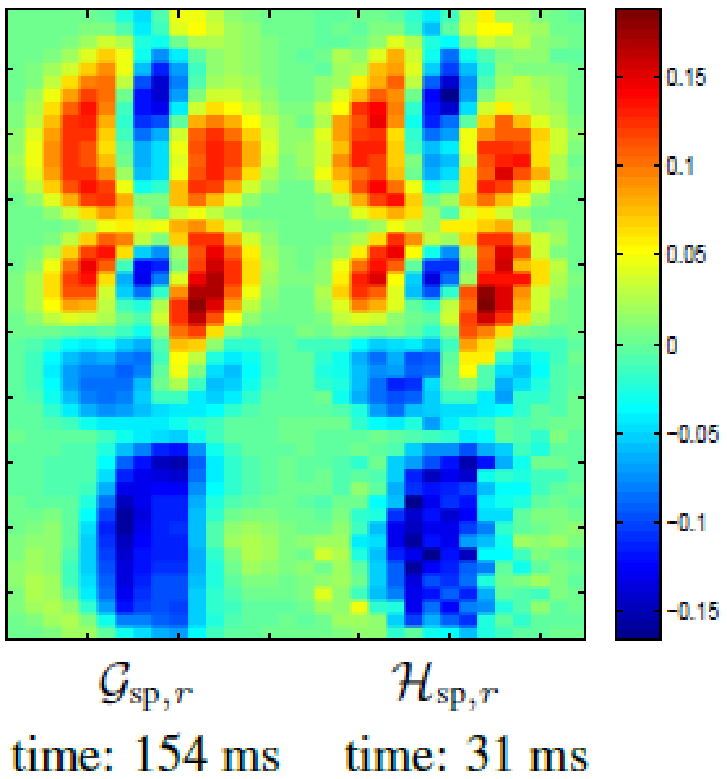}
\caption{$r=100\%$}
\end{subfigure}
\quad 
\begin{subfigure}[b]{0.22\textwidth}
\includegraphics[height=3.6cm,width=3.8cm]{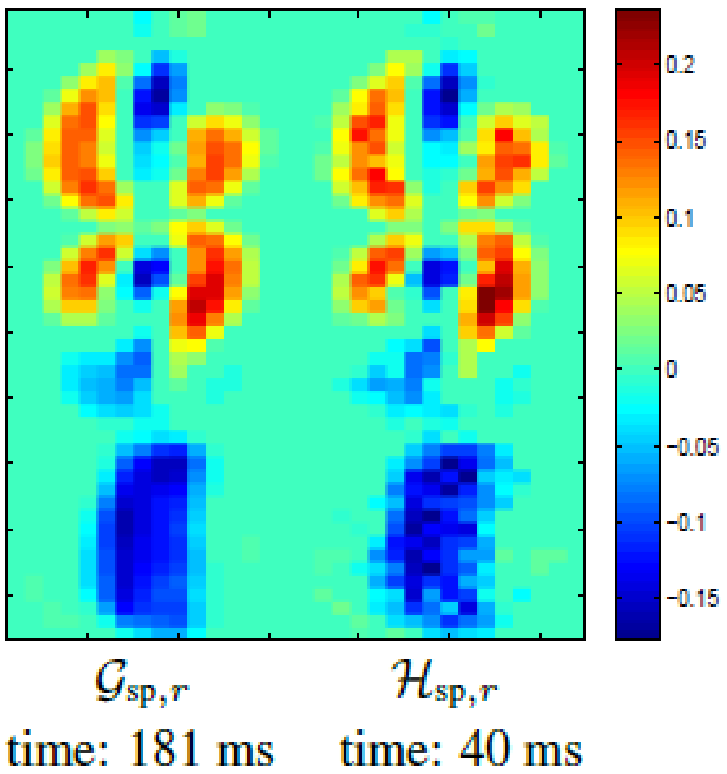}
\caption{$r=50\%$}
\end{subfigure}
\quad 
\begin{subfigure}[b]{0.22\textwidth}
\includegraphics[height=3.6cm,width=3.8cm]{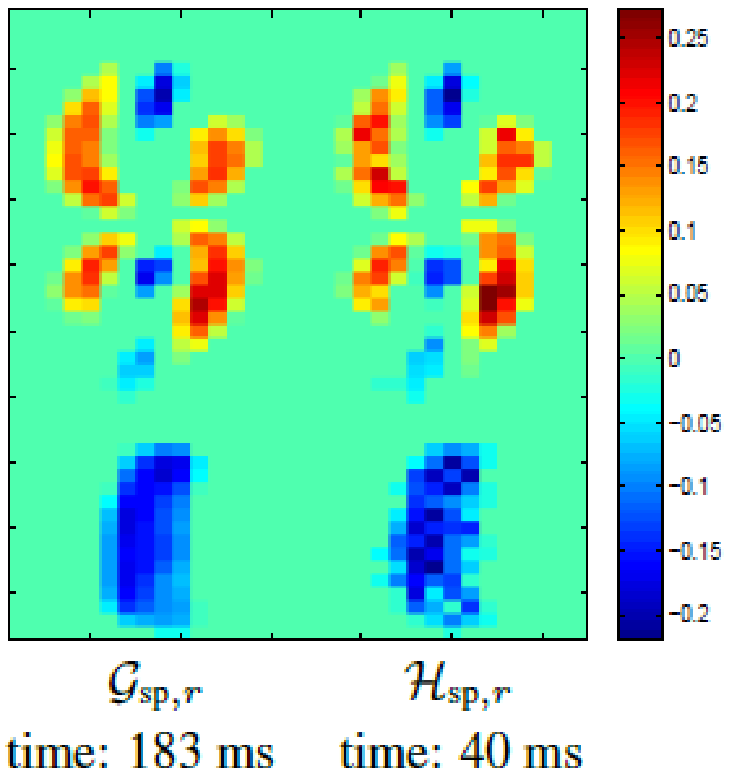}
\caption{$r=30\%$}
\end{subfigure}
\quad 
\begin{subfigure}[b]{0.22\textwidth}
\includegraphics[height=3.6cm,width=3.8cm]{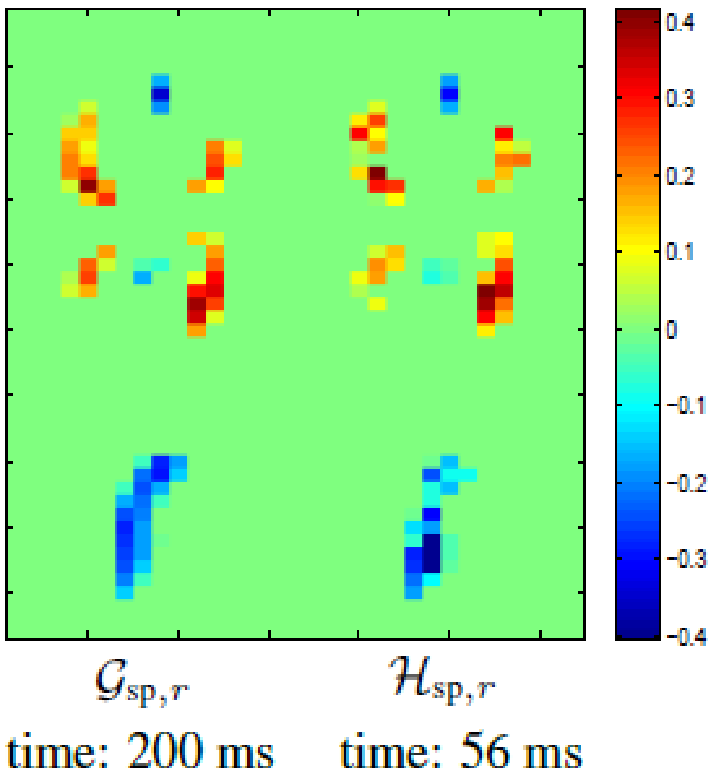}
\caption{$r=10\%$}
\end{subfigure}
\caption{[Digits] Visualization of top-3 sparse principal components.
In each figure, left panel shows $\mathcal{G}_{\text{sp},r}$ and right panel shows $\mathcal{H}_{\text{sp},r}$. 
\label{fig:digit_PC_viz}}
\end{figure}

\begin{figure*}[t]
\begin{center}
\begin{tabular}{cccc}
\resizebox{0.24\textwidth}{0.2\textwidth}{\includegraphics{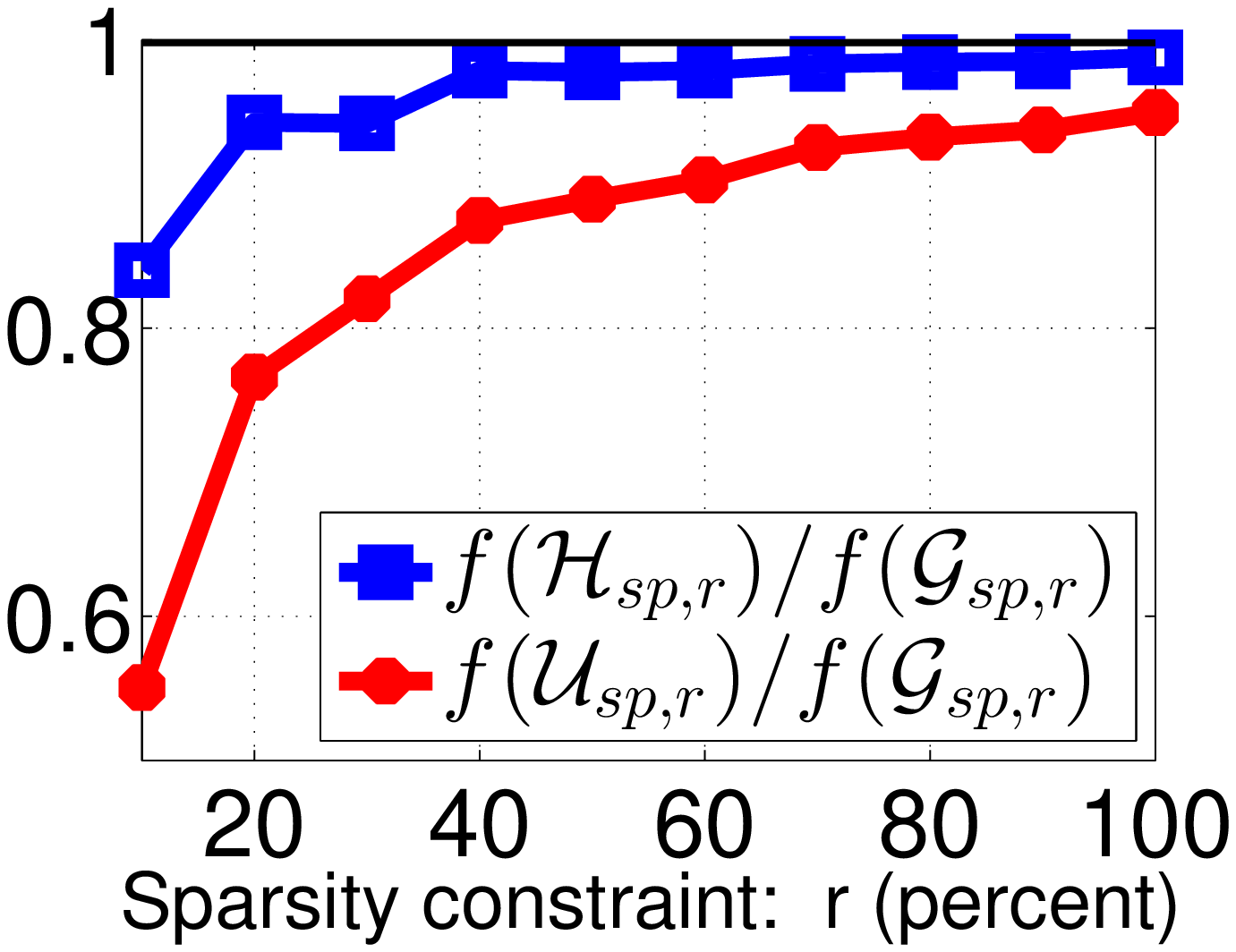}}&
\resizebox{0.23\textwidth}{0.2\textwidth}{\includegraphics{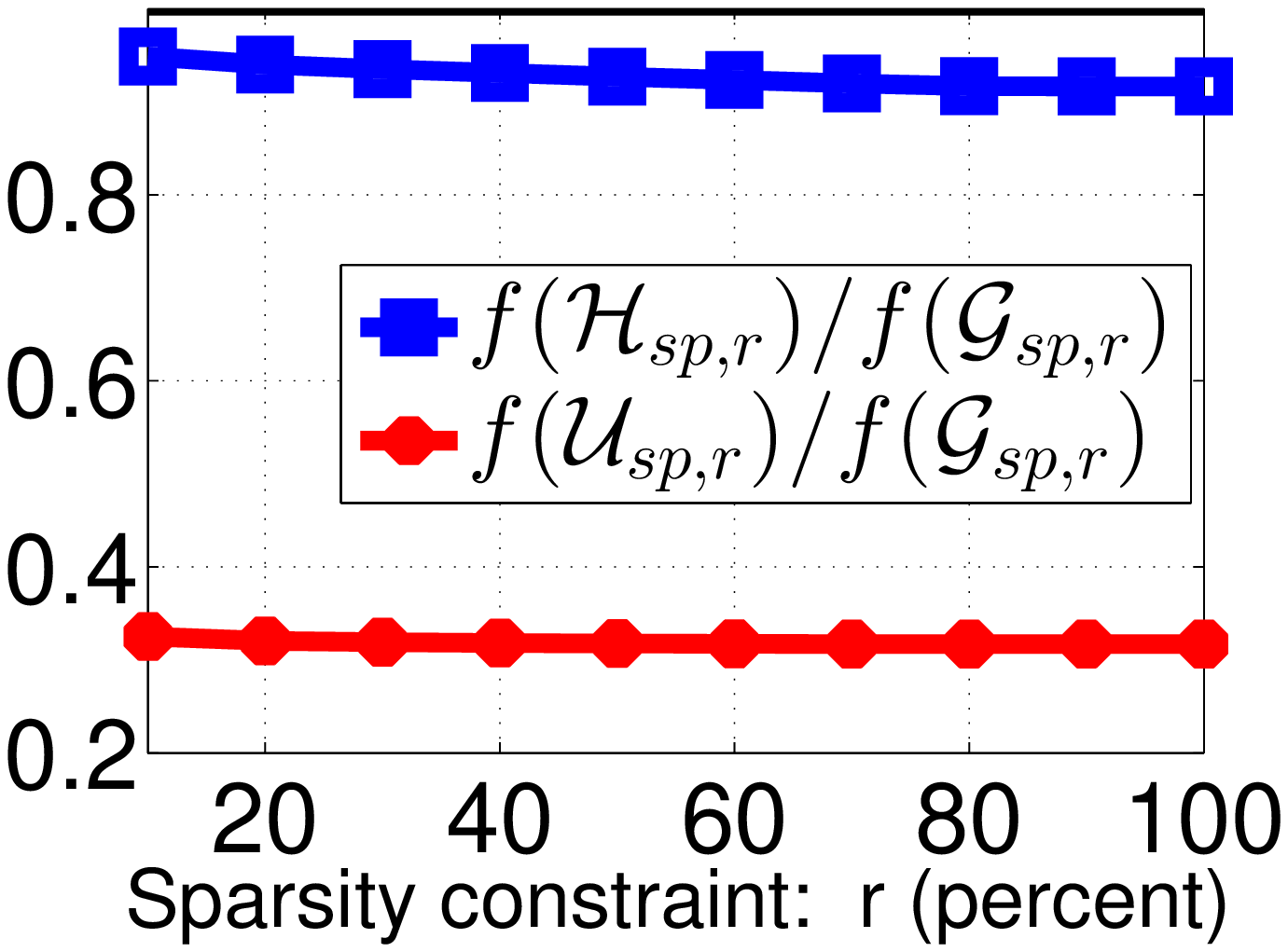}}&
%\\(a) Digits&(b) TectTC\\
\resizebox{0.23\textwidth}{0.2\textwidth}{\includegraphics{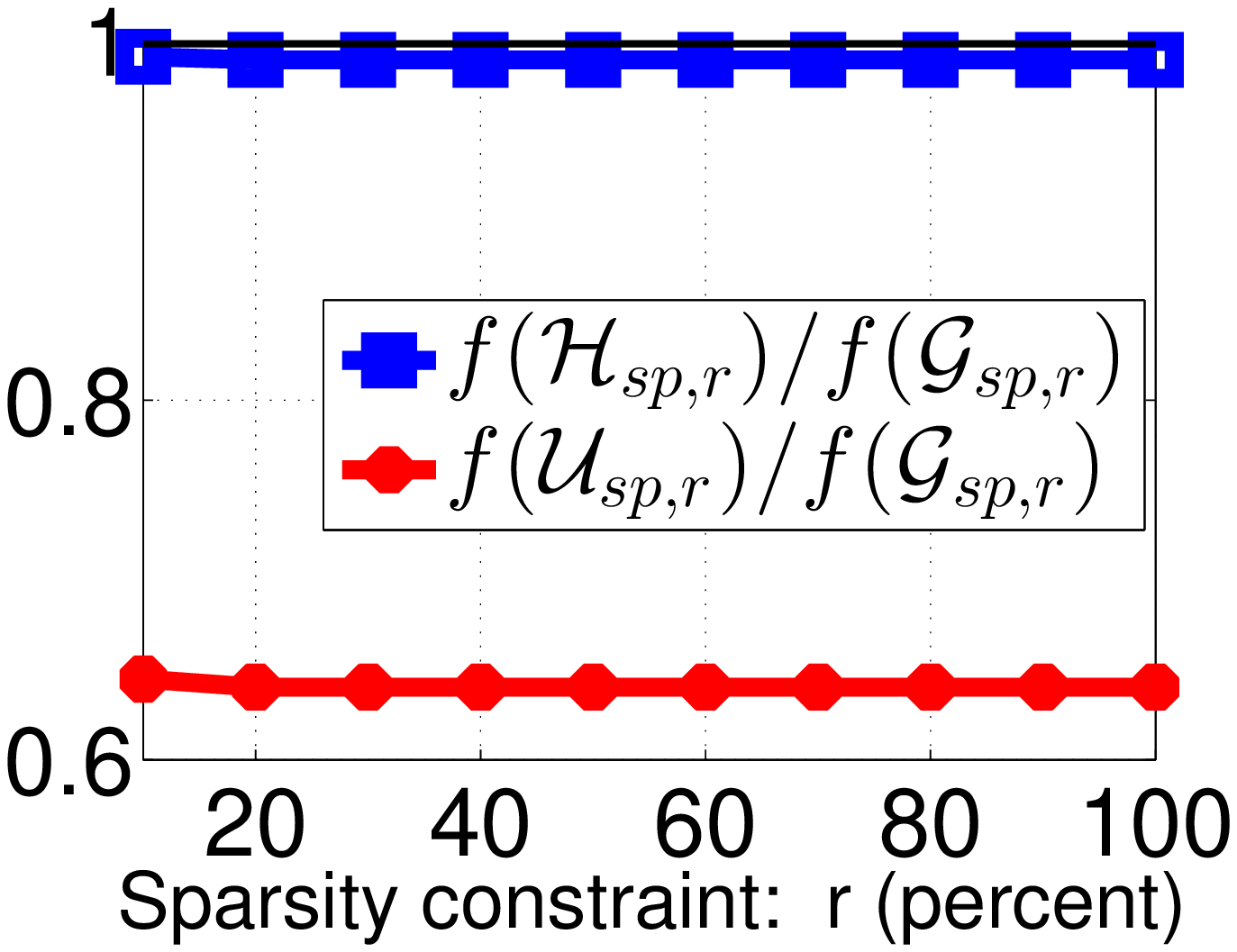}}
&
\resizebox{0.23\textwidth}{0.2\textwidth}{\includegraphics{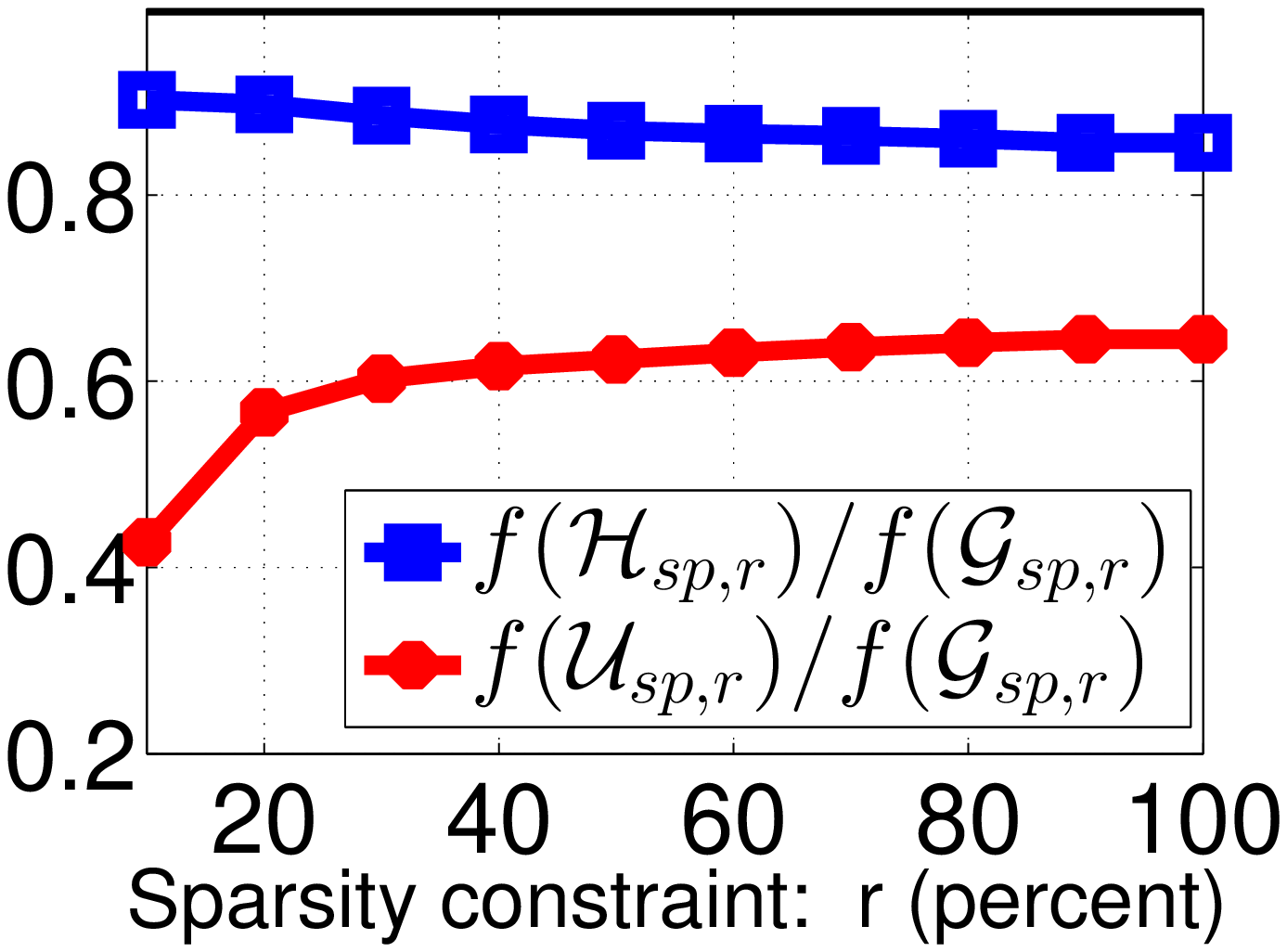}}
%\\(c) Stocks&(d) Gene Expression
\\
(a) Digit & (b) TechTC & (c) Stock & (d) Gene 
\end{tabular}
\end{center}
\caption{Performance of sparse PCA for \math{(\ell_1,\ell_2)}-sketch and
uniform sketch over an extensive range for the sparsity constraint \math{r}.
The performance of the uniform sketch is significantly worse
highlighting the importance of a good sketch.
\label{fig:ErrorAll}}
\end{figure*}

We show a visualization of the
principal components in Figure \ref{fig:digit_PC_viz}.
We observe that the sparse components from the \math{(\ell_1,\ell_2)}-sketch
are almost identical to the sparse components from the complete data.

%\subsubsection{TechTC Data}
\paragraph{TechTC Data.}
Algorithm~\ref{alg:alp*} gives optimal mixing parameter
$\alpha^* = 1$. We sample approximately 
$5\%$ of the elements from the centered data using  
our \math{(\ell_1,\ell_2)}-sampling,
as well as uniform sampling. 
The performance for small \math{r} is shown in
Table~\ref{fig:techtc_obj_ratio}, including the running time \math{\tau}.

\begin{table}[!h]
\begin{center}
\begin{tabular}{c|cc|cc}
$r$ & 
$\displaystyle\frac{f(\mathcal{H}_{\text{max/\bf sp}, r})}{f(\mathcal{G}_{\text{max/\bf sp}, r})}$&
\math{\displaystyle\frac{\tau(\cl{G})}{\tau(\cl{H})}}&
$\displaystyle\frac{f(\mathcal{U}_{\text{max/\bf sp}, r})}{f(\mathcal{G}_{\text{max/\bf sp}, r})}$&   
\math{\displaystyle\frac{\tau(\cl{G})}{\tau(\cl{U})}}\\[2pt]
&&&&\\[-9pt]\hline &&&&\\[-9pt]
20& 0.94/\bf 0.98&5.43&0.43/\bf 0.38&5.64\\
40& 0.94/\bf 0.99&5.70&0.41/\bf 0.38&5.96\\
60& 0.94/\bf 0.99&5.82&0.40/\bf 0.37&5.54\\
80& 0.93/\bf 0.99&5.55&0.39/\bf 0.37&5.24\\
100& 0.93/\bf 0.99&5.70&0.38/\bf 0.37&5.52\\
   \end{tabular}
\caption{
[TechTC]  Comparison of sparse principal components from the 
\math{(\ell_1,\ell_2)}-sketch and uniform sketch.
}
\label{fig:techtc_obj_ratio}
\end{center}
\end{table}

For this data, 
\math{f(\mathcal{G}_{\text{max}, r})/f(\mathcal{G}_{\text{sp}, r})\approx 0.84} ($r=10$).
We observe a very significant performance
 difference between the \math{(\ell_1,\ell_2)}-sketch and uniform sketch.
A more extensive comparison of recovered variance is given in
Figure~\ref{fig:ErrorAll}(b). 
We also observe a speed-up of a factor of about
6 for the \math{(\ell_1,\ell_2)}-sketch.
Unlike the digits data which is uniformly near \math{\pm1}, the text data
is ``spikey'' and now it is important to sample with a bias 
toward larger elements, which is why
the uniform-sketch performs very poorly.

As a final comparison, we look at the actual sparse top component with 
sparsity parameter \math{r=10}. The topic IDs in the TechTC data are
10567=''\textbf{US: Indiana: Evansville}''
and
11346=''\textbf{US: Florida}''.
The top-10 features (words) in the full PCA on the complete data are shown in
Table \ref{fig:techtc_words}.

\begin{table}[!h]
\begin{center}
\scalebox{1}{
    \begin{tabular}{| l | l || l | l |}
   \hline
\bf ID &\bf  Top 10 in $\mathcal{G}_{\text{max}, r}$& ID & Other words\\ \hline
\bf 1 \bf &\bf  evansville & 11 & service \\ \hline
\bf 2 &\bf  florida & 12 & small \\ \hline
\bf 3 &\bf  south & 13 & frame \\ \hline
\bf 4 &\bf  miami & 14 & tours \\ \hline
\bf 5 &\bf  indiana & 15 & faver \\ \hline
\bf 6 &\bf  information & 16 & transaction \\ \hline
\bf 7 &\bf  beach & 17 & needs \\ \hline
\bf 8 &\bf  lauderdale & 18 & commercial \\ \hline
\bf 9 &\bf  estate & 19 & bullet\\ \hline
\bf 10 &\bf  spacer & 20 & inlets \\ \hline
& & 21 & producer \\ \hline
   \end{tabular}
}
\caption
{
[TechTC] Top ten words in top principal component of the complete data
(the other words are discovered by some of the
sparse PCA algorithms).
}
\label{fig:techtc_words}
\end{center}
\end{table}

\begin{table}[!h]
\begin{center}
\scalebox{1}{
    \begin{tabular}{| c | c | c | c | c | c |}
   \hline
$\mathcal{G}_{\text{max}, r}$
& $\mathcal{H}_{\text{max}, r}$ 
& $\mathcal{U}_{\text{max}, r}$ 
& $\mathcal{G}_{\text{sp}, r}$ 
& $\mathcal{H}_{\text{sp}, r}$  
& $\mathcal{U}_{\text{sp}, r}$  \\ 
\hline \hline
1 & 1 & 6 & 1 & 1 & 6 \\ \hline
2 & 2 & 14 & 2 & 2 & 14 \\ \hline
3 & 3 & 15 & 3 & 3 & 15 \\ \hline
4 & 4 & 16 & 4 & 4 & 16 \\ \hline
5 & 5 & 17 & 5 & 5 & 17 \\ \hline
6 & 7 & 7 & 6 & 7 & 7 \\ \hline
7 & 6 & 18 & 7 & 8 & 18 \\ \hline
8 & 8 & 19 & 8 & 6 & 19 \\ \hline
9 & 11 & 20 & 9 & 12 & 20 \\ \hline
10 & 12 & 21 & 13 & 11 & 21 \\ \hline
   \end{tabular}
}
\caption
{
[TechTC] Relative ordering of the words (w.r.t. $\mathcal{G}_{\text{max}, r}$) in the top sparse principal component with sparsity parameter
 $r=10$.
}
\label{fig:techtc_words_ranked}
\end{center}
\end{table}

In Table~\ref{fig:techtc_words_ranked} we show which words appear in the 
top sparse principal component with sparsity \math{r=10} using various
sparse PCA algorithms.
We observe that the sparse PCA from the \math{(\ell_1,\ell_2)}-sketch 
with only 5\% of the data sampled 
matches quite closely with the same sparse PCA algorithm using the complete
data ($\mathcal{G}_{\text{max/sp}, r}$ matches $\mathcal{H}_{\text{max/sp}, r}$). 

%\subsubsection{Stock Data}
\paragraph{Stock Data.}
Algorithm~\ref{alg:alp*} gives optimal mixing parameter
$\alpha^* = 0.1$\footnote{we computed $\alpha^*$ numerically in the range $[0.1, 1]$.}.
We sample about 
$2\%$ of the non-zero elements from the centered data using  
our \math{(\ell_1,\ell_2)}-sampling,
as well as uniform sampling. 
The performance for small \math{r} is shown in
Table~\ref{fig:stock_obj_ratio}, including the running time \math{\tau}.

For this data, 
\math{f(\mathcal{G}_{\text{max}, r})/f(\mathcal{G}_{\text{sp}, r})\approx 0.96} ($r=10$).
We observe a very significant performance
 difference between the \math{(\ell_1,\ell_2)}-sketch and uniform sketch.
A more extensive comparison of recovered variance is given in
Figure~\ref{fig:ErrorAll}(c). 
We also observe a speed-up of a factor of about
4 for the \math{(\ell_1,\ell_2)}-sketch.
Similar to TechTC data this dataset is also ``spikey'',  and consequently biased sampling  
toward larger elements significantly outperforms the uniform-sketch.

We now look at the actual sparse top component with 
sparsity parameter \math{r=10}. 
The top-10 features (stocks) in the full PCA on the complete data are shown in
Table \ref{fig:stock_labels}.
In Table~\ref{fig:stock_labels_ranked} we show which stocks appear in the 
top sparse principal component using various
sparse PCA algorithms.
We observe that the sparse PCA from the \math{(\ell_1,\ell_2)}-sketch 
with only 2\% of the non-zero elements sampled 
matches quite closely with the same sparse PCA algorithm using the complete
data ($\mathcal{G}_{\text{max/sp}, r}$ matches $\mathcal{H}_{\text{max/sp}, r}$).

\begin{table}[!h]
\begin{center}
\begin{tabular}{c|cc|cc}
$r$ & 
$\displaystyle\frac{f(\mathcal{H}_{\text{max/\bf sp}, r})}{f(\mathcal{G}_{\text{max/\bf sp}, r})}$&
\math{\displaystyle\frac{\tau(\cl{G})}{\tau(\cl{H})}}&
$\displaystyle\frac{f(\mathcal{U}_{\text{max/\bf sp}, r})}{f(\mathcal{G}_{\text{max/\bf sp}, r})}$&   
\math{\displaystyle\frac{\tau(\cl{G})}{\tau(\cl{U})}}\\[2pt]
&&&&\\[-9pt]\hline &&&&\\[-9pt]
20& 1.00/\bf 1.00&3.85&0.69/\bf 0.67&4.74\\
40& 1.00/\bf 1.00&3.72&0.66/\bf 0.66&4.76\\
60& 0.99/\bf 0.99&3.86&0.65/\bf 0.66&4.61\\
80& 0.99/\bf 0.99&3.71&0.65/\bf 0.66&4.74\\
100& 0.99/\bf 0.99&3.63&0.64/\bf 0.65&4.71\\
   \end{tabular}
\caption{
[Stock data] Comparison of sparse principal components from the 
\math{(\ell_1,\ell_2)}-sketch and uniform sketch.
}
\label{fig:stock_obj_ratio}
\end{center}
\end{table}

\begin{table}[!h]
\begin{center}
\scalebox{1}{
    \begin{tabular}{| c | l || c | l |}
   \hline
\bf ID &\bf  Top 10 in $\mathcal{G}_{\text{max}, r}$& ID & Other stocks\\ \hline
\bf1 & \bf T.2 & 11 & HET. \\ \hline
\bf2 & \bf AIG & 12 & ONE.1 \\ \hline
\bf3 & \bf C & 13 & MA \\ \hline
\bf4 & \bf UIS & 14 & XOM \\ \hline
\bf5 & \bf NRTLQ & 15 & PHA.1 \\ \hline
\bf6 & \bf S.1 & 16 & CL \\ \hline
\bf7 & \bf GOOG & 17 & WY \\ \hline
\bf8 & \bf MTLQQ &  &  \\ \hline
\bf9 & \bf ROK &  & \\ \hline
\bf10 &\bf EK &  &  \\ \hline
   \end{tabular}
}
\caption
{
[Stock data] Top ten stocks in top principal component of the complete data
(the other stocks are discovered by some of the
sparse PCA algorithms).
}
\label{fig:stock_labels}
\end{center}
\end{table}
%  ===============================================
%  ===============================================
\begin{table}[!h]
\begin{center}
\scalebox{1}{
    \begin{tabular}{| c | c | c | c | c | c |}
   \hline
$\mathcal{G}_{\text{max}, r}$
& $\mathcal{H}_{\text{max}, r}$ 
& $\mathcal{U}_{\text{max}, r}$ 
& $\mathcal{G}_{\text{sp}, r}$ 
& $\mathcal{H}_{\text{sp}, r}$  
& $\mathcal{U}_{\text{sp}, r}$  \\ 
\hline \hline
1 & 1 & 2 & 1 & 1 & 2 \\ \hline
2 & 2 & 11 & 2 & 2 & 11 \\ \hline
3 & 3 & 12 & 3 & 3 & 12 \\ \hline
4 & 4 & 13 & 4 & 4 & 13 \\ \hline
5 & 5 & 14 & 5 & 5 & 14 \\ \hline
6 & 6 & 3 & 6 & 7 & 3 \\ \hline
7 & 7 & 15 & 7 & 6 & 15 \\ \hline
8 & 9 & 9 & 8 & 8 & 9 \\ \hline
9 & 8 & 16 & 9 & 9 & 16 \\ \hline
10 & 11 & 17 & 10 & 11 & 17 \\ \hline
   \end{tabular}
}
\caption
{
[Stock data]  Relative ordering of the stocks (w.r.t. $\mathcal{G}_{\text{max}, r}$) 
in the top sparse principal component with sparsity parameter
 $r=10$.
}
\label{fig:stock_labels_ranked}
\end{center}
\end{table}

%\subsubsection{Gene Expression Data}
\paragraph{Gene Expression Data.}
Algorithm~\ref{alg:alp*} gives optimal mixing parameter
$\alpha^* = 0.92$. We sample about 
$9\%$ of the elements from the centered data using  
our \math{(\ell_1,\ell_2)}-sampling,
as well as uniform sampling. 
The performance for small \math{r} is shown in
Table~\ref{fig:gene_obj_ratio}, including the running time \math{\tau}.

For this data, 
\math{f(\mathcal{G}_{\text{max}, r})/f(\mathcal{G}_{\text{sp}, r})\approx 0.05} ($r=10$) 
which means a good sparse PCA algorithm is imperative.
We observe a very significant performance
 difference between the \math{(\ell_1,\ell_2)}-sketch and uniform sketch.
A more extensive comparison of recovered variance is given in
Figure~\ref{fig:ErrorAll}(d). 
We also observe a speed-up of a factor of about
4 for the \math{(\ell_1,\ell_2)}-sketch.
Similar to TechTC data this dataset is also ``spikey'',  and consequently biased sampling  
toward larger elements significantly outperforms the uniform-sketch.

Also, we look at the actual sparse top component with 
sparsity parameter \math{r=10}. 
The top-10 features (probes) in the full PCA on the complete data are shown in
Table \ref{fig:probe_labels}.

In Table~\ref{fig:probe_labels_ranked} we show which probes appear in the 
top sparse principal component with sparsity \math{r=10} using various
sparse PCA algorithms.
We observe that the sparse PCA from the \math{(\ell_1,\ell_2)}-sketch 
with only 9\% of the elements sampled 
matches reasonably with the same sparse PCA algorithm using the complete
data ($\mathcal{G}_{\text{max/sp}, r}$ matches $\mathcal{H}_{\text{max/sp}, r}$). 

\iffalse %===========
We sample (and rescale) $\sim 9\%$ elements from centered data via optimal hybrid-$(\ell_1,\ell_2)$ distribution ($\alpha^*=0.92$) and uniform distribution to produce respective sparse sketches.
Table \ref{fig:gene_obj_ratio} shows that $f(\mathcal{H}_{\text{sp},r})$ follows $f(\mathcal{G}_{\text{sp},r})$, and $f(\mathcal{H}_{\text{max},r})$ is reasonably close to $f(\mathcal{G}_{\text{max},r})$. This indicates that only $9\%$ elements sampled via optimal hybrid sampling reasonably preserve the linear structure of this  data irrespective of the choice of algorithm (\textit{`Spasm'} is slightly better). Moreover, computation of $\mathcal{H}_{\text{max},r}$ and $\mathcal{H}_{\text{sp},r}$ is significantly faster than that of $\mathcal{G}_{\text{max},r}$ and $\mathcal{G}_{\text{sp},r}$ (Table \ref{fig:gene_time_ratio}).
\fi

\begin{table}[t]
\begin{center}
\begin{tabular}{c|cc|cc}
$r$ & 
$\displaystyle\frac{f(\mathcal{H}_{\text{max/\bf sp}, r})}{f(\mathcal{G}_{\text{max/\bf sp}, r})}$&
\math{\displaystyle\frac{\tau(\cl{G})}{\tau(\cl{H})}}&
$\displaystyle\frac{f(\mathcal{U}_{\text{max/\bf sp}, r})}{f(\mathcal{G}_{\text{max/\bf sp}, r})}$&   
\math{\displaystyle\frac{\tau(\cl{G})}{\tau(\cl{U})}}\\[2pt]
&&&&\\[-9pt]\hline &&&&\\[-9pt]
20& 0.82/\bf 0.81&3.76&0.64/\bf 0.16&2.57\\
40& 0.82/\bf 0.88&3.61&0.65/\bf 0.15&2.53\\
60& 0.83/\bf 0.90&3.86&0.67/\bf 0.10&2.85\\
80& 0.84/\bf 0.94&3.71&0.68/\bf 0.11&2.85\\
100& 0.84/\bf 0.91&3.78&0.67/\bf 0.10&2.82\\
   \end{tabular}
\caption{
[Gene data] 
Comparison of sparse principal components from the 
\math{(\ell_1,\ell_2)}-sketch and uniform sketch.
}
\label{fig:gene_obj_ratio}
\end{center}
\end{table}

\begin{table}[t]
\begin{center}
\scalebox{1}{
    \begin{tabular}{| c | l || c | l |}
   \hline
\bf ID &\bf  Top 10 in $\mathcal{G}_{\text{max}, r}$ & ID & Other probes\\ \hline
\bf 1 & \bf 210081\_at & 11 & 205866\_at \\ \hline
\bf 2 & \bf 214387\_x\_at & 12 & 209074\_s\_at \\ \hline
\bf 3 & \bf 211735\_x\_at & 13 & 205311\_at \\ \hline
\bf 4 & \bf 209875\_s\_at & 14 & 216379\_x\_at \\ \hline
\bf 5 & \bf 205982\_x\_at & 15 & 203571\_s\_at \\ \hline
\bf 6 & \bf 215454\_x\_at & 16 & 205174\_s\_at \\ \hline
\bf 7 & \bf 209613\_s\_at & 17 & 204846\_at \\ \hline
\bf 8 & \bf 210096\_at & 18 & 209116\_x\_at \\ \hline
\bf 9 & \bf 204712\_at & 19 & 202834\_at \\ \hline
\bf 10 &\bf  203980\_at & 20 & 209425\_at \\ \hline
&& 21 & 215356\_at\\ \hline
&& 22 & 221805\_at\\ \hline
&& 23 & 209942\_x\_at\\ \hline
&& 24 & 218450\_at \\ \hline
&& 25 & 202508\_s\_at \\ \hline
   \end{tabular}
}
\caption
{
[Gene data] 
Top ten probes in top principal component of the complete data
(the other probes are discovered by some of the
sparse PCA algorithms).
}
\label{fig:probe_labels}
\end{center}
\end{table}

Finally, we validate the genes corresponding to the top probes in the context of lung cancer. Table \ref{fig:gene_labels_ranked} lists the top 
twelve gene symbols in Table \ref{fig:probe_labels}. Note that a gene can occure
multiple times in principal component since genes can 
be associated with different
probes.

\begin{table}[!h]
\begin{center}
\scalebox{1}{
    \begin{tabular}{| c | c | c | c | c | c |}
   \hline
$\mathcal{G}_{\text{max}, r}$
& $\mathcal{H}_{\text{max}, r}$ 
& $\mathcal{U}_{\text{max}, r}$ 
& $\mathcal{G}_{\text{sp}, r}$ 
& $\mathcal{H}_{\text{sp}, r}$  
& $\mathcal{U}_{\text{sp}, r}$  \\ 
\hline \hline
1  & 4   & 13  & 1	& 4 	& 13 \\ \hline
2  & 1   & 14  & 2	& 1 	& 16 \\ \hline
3  & 11 & 3    & 3 	& 2 	& 15 \\ \hline
4  & 2   & 15  & 4 	& 11 	& 19 \\ \hline
5  & 3   & 5  	& 5 	& 3 	& 20 \\ \hline
6  & 8   & 16  & 6 	& 8 	& 21 \\ \hline
7  & 7   & 6  	& 7 	& 7 	& 22 \\ \hline
8  & 9   & 17  & 8 	& 9 	& 23 \\ \hline
9  & 5   & 4    & 9 	& 5 	& 24 \\ \hline
10 & 12 & 18 & 10 	& 12 	& 25 \\ \hline
   \end{tabular}
}
\caption
{
[Gene data] Relative ordering of the probes (w.r.t. $\mathcal{G}_{\text{max}, r}$)
in the top sparse principal component with sparsity parameter
 $r=10$.
}
\label{fig:probe_labels_ranked}
\end{center}
\end{table}

\begin{table}[!h]
\begin{center}
\scalebox{1}{
    \begin{tabular}{| l | l | l | l | l | l |}
   \hline
$\mathcal{G}_{\text{max}, r}$
& $\nu$ 
& $\mathcal{H}_{\text{max}, r}$ 
& $\nu$
& $\mathcal{H}_{\text{sp}, r}$  
& $\nu$ \\
\hline \hline
SFTPC		& 4 	& SFTPC 	& 3 &SFTPC	& 3 \\ \hline
AGER 		& 1 	& SPP1 	& 1 &SPP1	& 1  \\ \hline
SPP 1 		& 1 	& AGER 	& 1 &AGER	& 1  \\ \hline
ADH1B 	& 1 	& FCN3 	& 1 & FCN3 	& 1 \\ \hline
CYP4B1 	& 1 	& CYP4B1 	& 1 &CYP4B1	& 1 \\ \hline
WIF1 		& 1 	& ADH1B 	& 1 &ADH1B	& 1 \\ \hline
FABP4 	& 1 	& WIF1 	& 1 &WIF1	& 1 \\ \hline
		&	& FAM107A 	& 1 &FAM107A& 1 \\ \hline
   \end{tabular}
}
\caption
{
[Gene data] Gene symbols corresponding to top probes in Table \ref{fig:probe_labels_ranked}. One gene can be associated with multiple probes. Here $\nu$ is the frequency of occurrence of a gene in top ten probes of their respective principal component.
}
\label{fig:gene_labels_ranked}
\end{center}
\end{table}

Genes like SFTPC, AGER, WIF1, and FABP4 are down-regulated in lung cancer, 
while SPP1 is up-regulated (see the functional gene grouping at: \url{www.sabiosciences.com/rt_pcr_product/HTML/PAHS-134Z.html}).
Co-expression analysis on the set of eight genes for $\mathcal{H}_{\text{max},r}$ and $\mathcal{H}_{\text{sp},r}$ using the tool ToppFun (\url{toppgene.cchmc.org/}) shows that all eight genes appear in a list of selected probes characterizing non-small-cell lung carcinoma (NSCLC) in \cite[Table S1]{Hou10}. Further, 
AGER and FAM107A appear in the \textit{top five} highly discriminative genes in \cite[Table S3]{Hou10}.
Additionally, AGER, FCN3, SPP1, and ADH1B appear among the 162 most differentiating genes across two subtypes of NSCLC and normal lung cancer in 
\cite[Supplemental Table 1]{Dracheva07}.
Such findings show that our method can identify, from incomplete
data, important genes for complex diseases like cancer. Also, notice that our sampling-based method is able to identify additional important genes, such as, FCN3 and FAM107A in top ten genes.

\subsection{Performance of Other Sketches}

We briefly report on 
other options for sketching \math{\matA}. 
First, we consider suboptimal \math{\alpha} (not \math{\alpha^*} from 
Algorithm~\ref{alg:alp*}) in (\ref{eqn:hybrid_prob}) to construct a suboptimal hybrid distribution. 
We use this distribution in proto-Algorithm \ref{alg:proto_alg1} to construct a sparse sketch. 
Figure~\ref{fig:ErrorOther_subalpha} reveals that a good sketch using the
optimal \math{\alpha^*} is important.

\begin{figure}[!th]
\begin{center}
\includegraphics[scale=0.4]{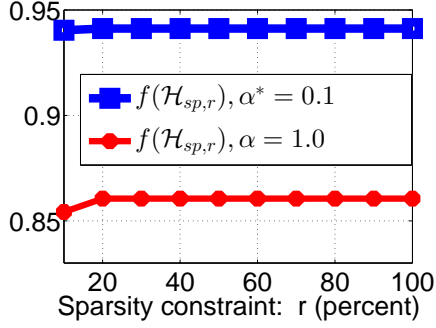}
\end{center}
\caption{[Stock data] Performance of sketch using \textit{suboptimal} $\alpha$ to illustrate 
the importance of the optimal mixing parameter 
\math{\alpha^*}.
\label{fig:ErrorOther_subalpha}}
\end{figure}

Second, another popular sketching method using element wise 
sparsification is to sample elements not biasing toward larger elements but 
rather toward elements whose leverage scores are high.
See \cite{BCSW14} for the detailed form of the leverage score sampling probabilities
(which are known to work well in other settings
can be plugged into our proto-Algorithm~\ref{alg:proto_alg1}). 
Let $\matA$ be a $m \times n$ matrix of rank $\rho$, and its SVD is given by $\matA=\matU\matSig\matV^T$. 
Then, we define $\mu_i$ (row leverage scores), $\nu_j$ (column leverage scores), and element-wise leverage scores
$p_{lev}$ as follows:
\begin{eqnarray*}
&& \mu_i = \TsNorm{\matU_{(i)}}^2, \quad  \nu_j = \TsNorm{\matV_{(j)}}^2,\\
&&p_{lev} = \frac{1}{2}\cdot \frac{\mu_i+\nu_j}{(m+n)\rho} + \frac{1}{2mn}, \quad i\in[m], j \in [n]
\end{eqnarray*}

At a high level, the leverage score of element \math{(i,j)} is proportional
to the squared norms of the \math{i}th row of the left singular matrix and the
\math{j}th row of the right singular matrix.
Such leverage score sampling is different from uniform 
sampling only for \textit{low rank} matrices or 
\textit{low rank approximations} to matrices, so we used a low
rank approximation to the data matrix. 
We construct such low-rank approximation by projecting a dataset onto a low dimensional subspace.
We notice that the datasets projected onto the space spanned by top few principal components 
preserve the linear structure of the data. For example, Digit data show good separation 
of digits when projected onto the top three PCA's. For TechTC and Gene data the top two respective 
PCA's are good enough to form a low-dimensional subspace where the datasets show reasonable separation
of two classes of samples. For the stock data we use top three PCA's because the stable rank is 
close to 2. 

Let $\mathcal{L}_{\text{sp},r}$ be the \math{r}-sparse components using \textit{Spasm} for 
the leverage score sampled sketch \math{\tilde\matA}. 
Figure \ref{fig:ErrorOther} shows that leverage score sampling is 
not as effective as the optimal hybrid  
\math{(\ell_1,\ell_2)}-sampling 
for sparse PCA of low-rank data.
 
\begin{figure}[t]
\begin{center}
\begin{tabular}{cccc}
%\resizebox{0.33\textwidth}{!}{\includegraphics{results/stock/SPCA_stock_obj_sp_alpha2.eps}}
%& \qquad 
\resizebox{0.23\textwidth}{!}{\includegraphics{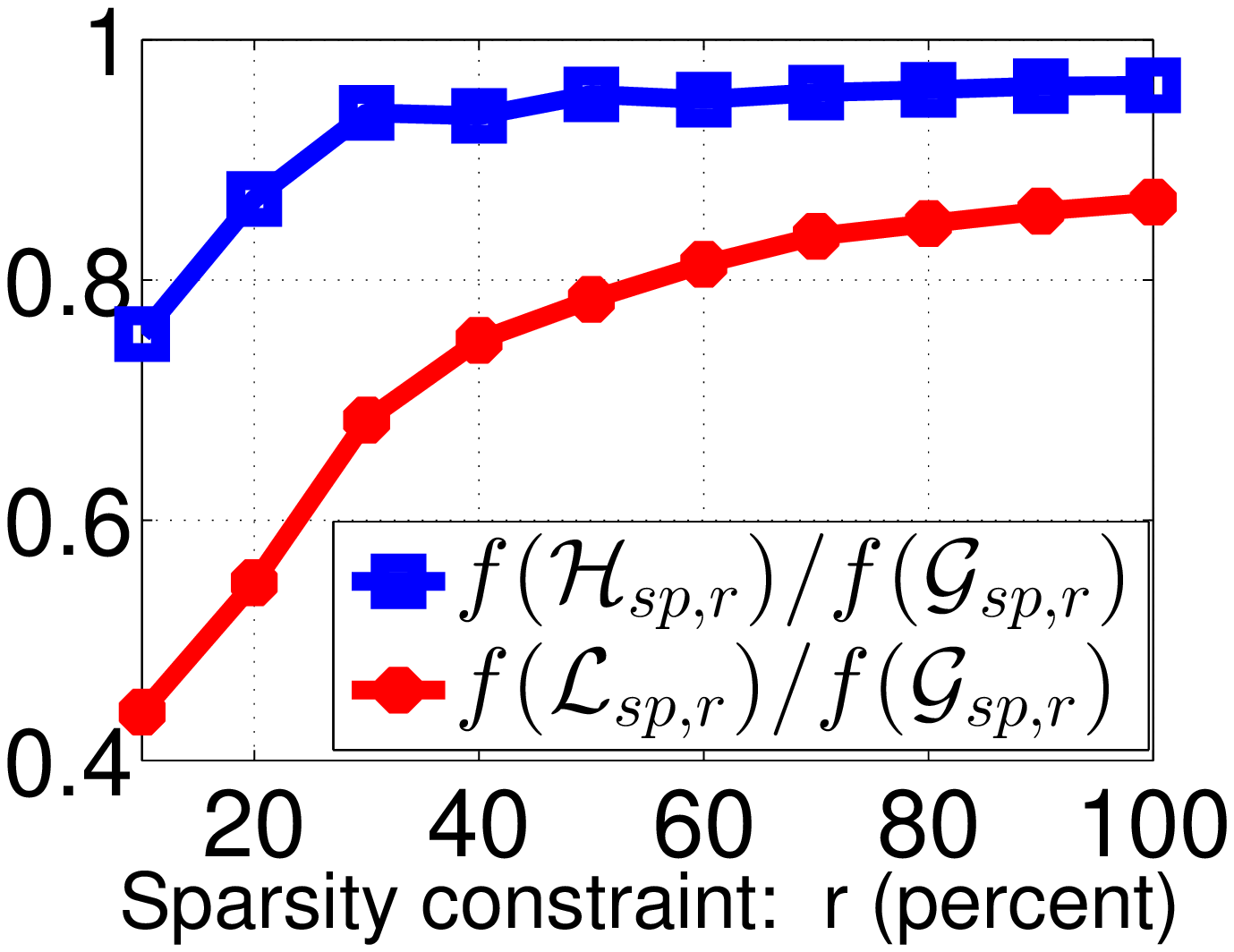}}&
\resizebox{0.23\textwidth}{!}{\includegraphics{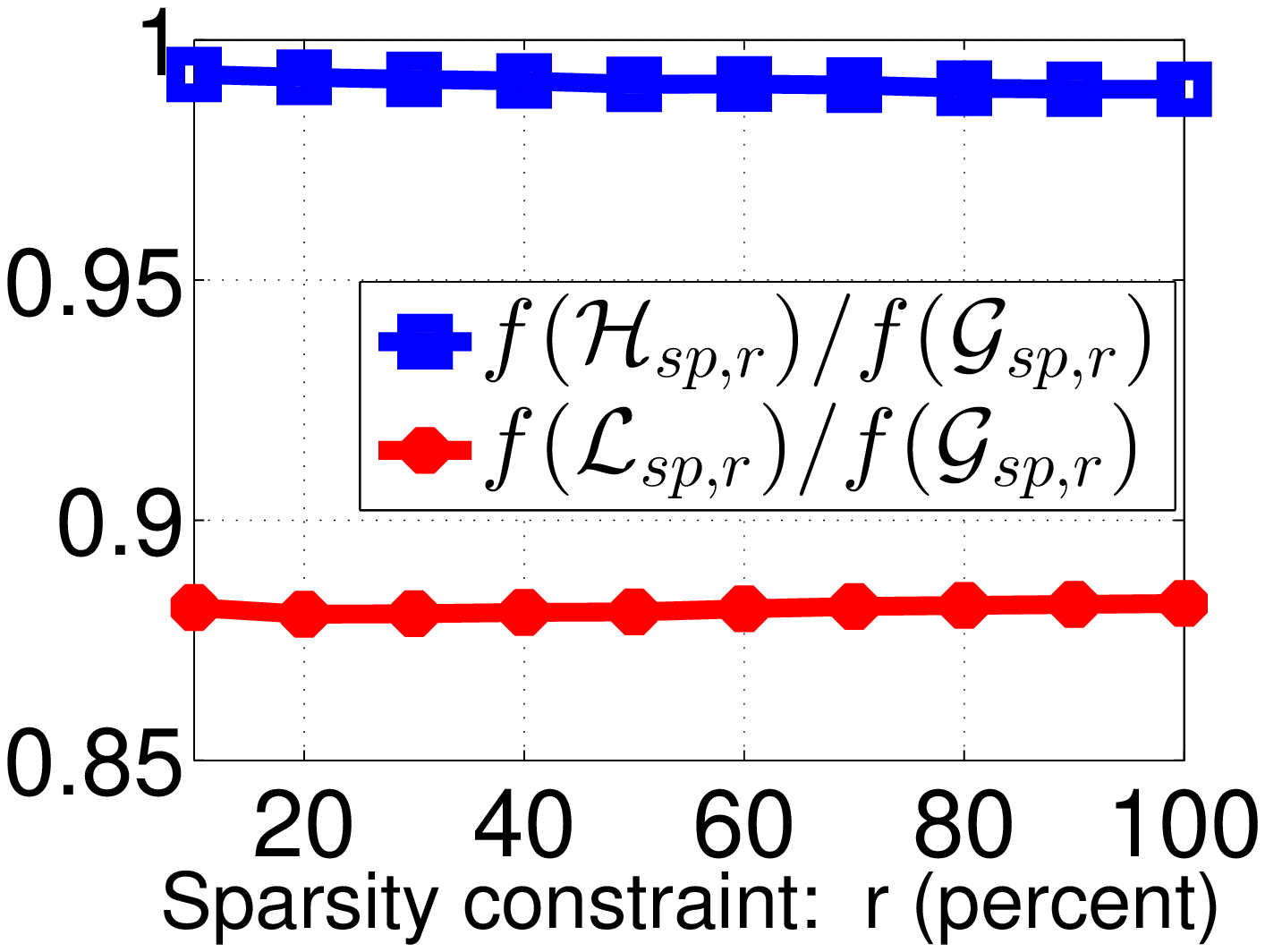}}&
\resizebox{0.23\textwidth}{!}{\includegraphics{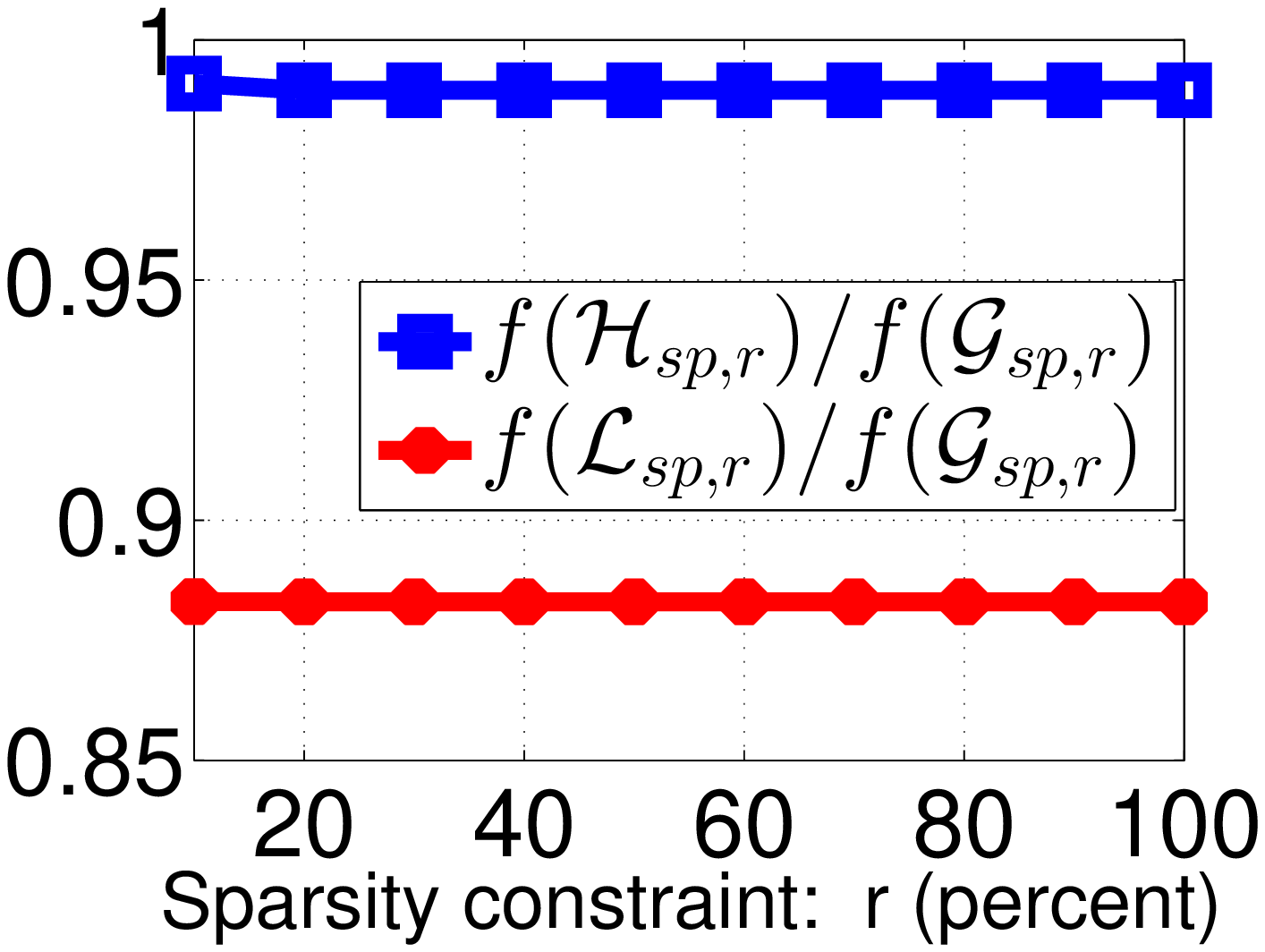}}&
\resizebox{0.23\textwidth}{!}{\includegraphics{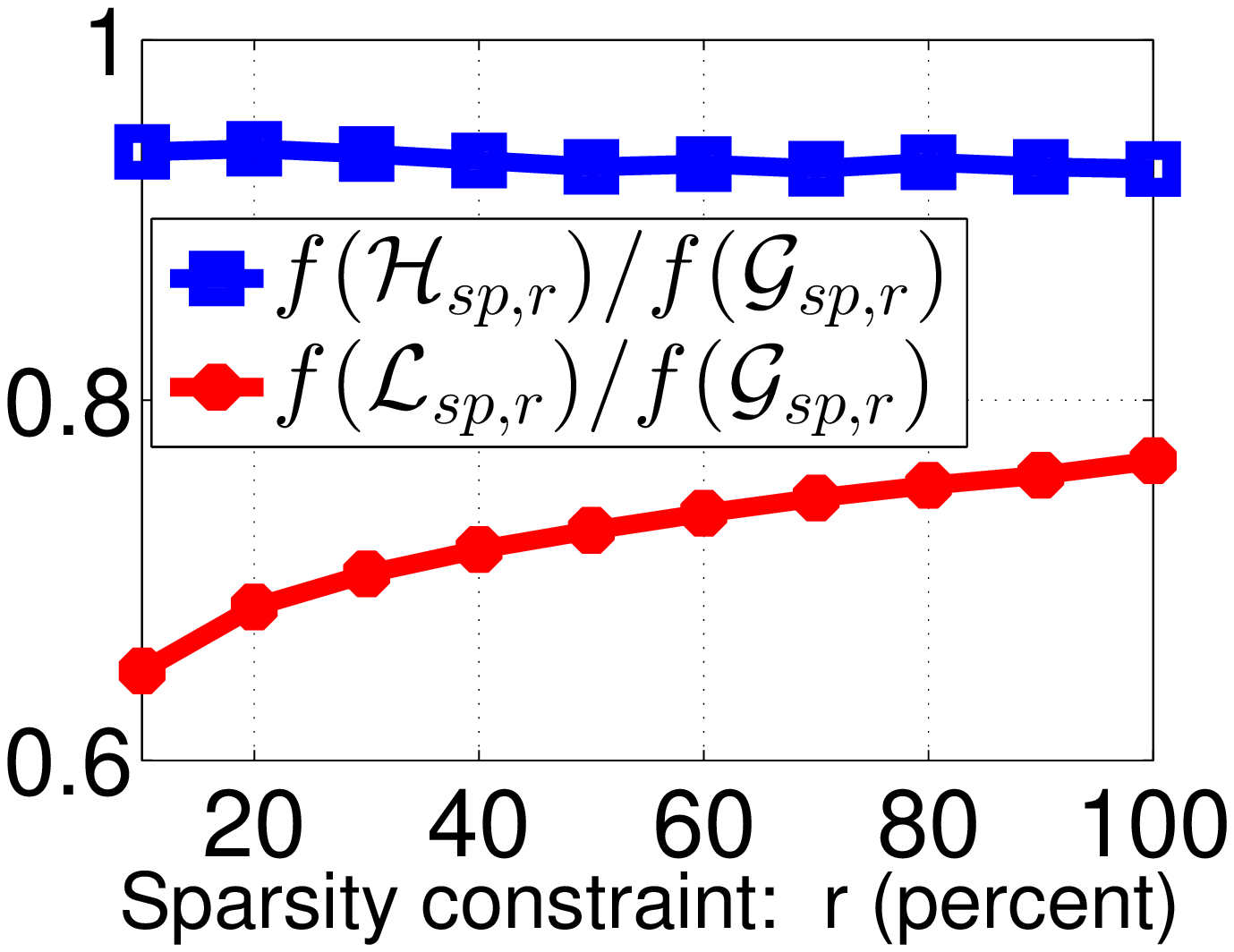}}
\\
(a) Digit (rank 3) & (b) TechTC (rank 2) & (c) Stock (rank 3) & (d) Gene (rank 2)
\end{tabular}
\end{center}
\caption{[Low-rank data] Performance of sparse PCA of low-rank data for optimal \math{(\ell_1,\ell_2)}-sketch and
leverage score sketch over an extensive range for the sparsity constraint \math{r}.
The performance of the optimal hyrbid sketch is considerably better
highlighting the importance of a good sketch.
%\caption{[Low-rank data] Comparison between optimal hybrid sketches and leverage score based sketches for low-rank data methods to illustrate 
\label{fig:ErrorOther}}
\end{figure}

\paragraph{Conclusion.}
It is possible to use a sparse sketch (incomplete data)
to recover nearly as good sparse 
principal components as you would have gotten with the complete data.
We mention that,
while \math{\cl{G}_{\text{max}}} which uses the largest weights
in the unconstrained PCA does not perform well with respect to the variance,
it does identify good features. A simple enhancement to
 \math{\cl{G}_{\text{max}}} is to recalibrate the sparse
component after identifying
the features - this is an unconstrained PCA problem on just the columns of
the data matrix corresponding to the features. This method of recalibrating
can be used to improve any sparse PCA algorithm.

Our algorithms are simple and efficient, and many interesting
avenues for further research remain.
Can the sampling complexity for the top-\math{k} sparse PCA be reduced from
\math{O(k^2)} to \math{O(k)}. We suspect that this should be possible
by getting a better bound on 
\math{\sum_{i=1}^k\sigma_i(\matA\transp\matA-\tilde\matA\transp\tilde\matA)};
we used the crude 
bound \math{k\norm{\matA\transp\matA-\tilde\matA\transp\tilde\matA}_2}.
We also presented a general surrogate optimization bound
which may be of interest in other applications. In particular,
it is pointed out in \cite{MB2015} that though PCA optimizes variance,
a more natural way to look at PCA is as the linear projection of the 
data that minimizes the \emph{information loss}. 
\cite{MB2015} gives efficient algorithms to find sparse 
linear dimension reduction that minimizes information loss -- the
information loss of sparse PCA can be considerably higher than optimal.
To minimize information loss, the objective to maximize is 
\math{f(\matV)=\trace(\matA\transp\matA\matV(\matA\matV)^\dagger\matA)}.
It would be interesting to see whether one can recover 
sparse low-information-loss linear projectors from incomplete data.

%---------------------------------------------------------------
\clearpage
\bibliography{sparsification,spca} \bibliographystyle{plainnat}
%---------------------------------------------------------------

\end{document}